\documentclass[
reprint,
 amsmath,amssymb,
 aps,
 prl
]{revtex4-2}
\usepackage[dvipdfmx]{graphicx}
\usepackage{dcolumn}
\usepackage{bm}

\newcommand{\Msum}{M_{\rm sum}}
\newcommand{\Msumu}{ M_{\mathrm{sum},u} }
\newcommand{\Msumv}{ M_{\mathrm{sum},v} }

\begin{document}

\preprint{APS/123-QED}

\title{How noise affects memory in linear recurrent networks}



\author{JingChuan Guan}
\email{kan@isi.imi.i.u-tokyo.ac.jp}

\author{Tomoyuki Kubota}

\author{Yasuo Kuniyoshi}

\author{Kohei Nakajima}

\affiliation{%
Intelligent Systems and Informatics Laboratory, Graduate School of Information Science and Technology, The University of Tokyo, 7-3-1 Hongo, Bunkyo-Ku, Tokyo, Japan.
}

\date{\today}

\begin{abstract}
The effects of noise on memory in a linear recurrent network are theoretically investigated.
Memory is characterized by its ability to store previous inputs in its instantaneous state of network, which receives a correlated or uncorrelated noise.
Two major properties are revealed:
First, the memory reduced by noise is uniquely determined by the noise's power spectral density (PSD). 
Second, the memory will not decrease regardless of noise intensity if the PSD is in a certain class of distribution (including power law).
The results are verified using the human brain signals, showing good agreement.

\begin{description}
\item[Keywords]
noise, memory, information processing, RNN, autocorrelation.
\end{description}
\end{abstract}

\maketitle

\textit{Introduction}.---Understanding the effects of noise on information processing is a crucial problem in comprehending any physical system.
For instance, in the field of quantum computation, 
the interaction between a quantum device and environment occurs with noise in the device, which impairs the accuracy of quantum computation~\cite{google2023suppressing, georgescu202025}. 
In nature, living organisms process information from the external environment
by extracting the necessary inputs from a large amount of signals containing noise, where noise works as a type of disturbance.

Short-term memory plays an essential role among various types of information processing,
which requires past input history.
This includes various tasks required in daily lives:
mental calculation~\cite{jonides2008mind}, recalling brief number of items~\cite{roediger20171}, and motor controls involving precise time perception~\cite{buonomano2010population, buonomano2017your}.
Additionally, many recent studies have reported that various types of physical systems can be utilized as computational resources~\cite{nakajima2020physical, nakajima2021reservoir}
where their short-term memories are exploited to solve tasks~\cite{liang2024physical}.
In the theoretical studies exploring recurrent neural networks (RNNs),
the memory has been characterized by memory function (MF)~\cite{jaeger2001short}
and information processing capacity~\cite{dambre2012information}, which can comprehensively reveal the memory in the network
\cite{faisal2008noise, moss2004stochastic, white2004short,
ganguli2008memory, rodan2010minimum, farkavs2016computational, gonon2020memory, toyoizumi2011beyond, schuecker2018optimal, haruna2019optimal, kubota2021unifying, ballarin2023memory}.
Using these measures, 
the dependency of memory on parameters has been investigated. 
Some studies~\cite{rodan2010minimum, kawai2019small, dale2021reservoir}
numerically revealed that its network topology~\cite{erdds1959random, watts1998collective, barabasi1999emergence} affects the memory.
From the perspective of noise, other researches~\cite{jaeger2001short,white2004short,dambre2012information} have reported that random noise reduces the past inputs held in the network and that,
as the noise-to-signal ratio (NSR) increases,
the reduction becomes more critical.
Therefore, the random noise dominantly has the negative impact on the information processing based on the short-term memory.

Those researches~\cite{jaeger2001short,white2004short,dambre2012information} have focused on the case of random noise,
which is termed independent and identically distributed (i.i.d.) noise;
however, real-world systems receive not only i.i.d. noise but also correlated noise. 
For example, $1/f$-like noise~\cite{voss1975f, west1990noise, voss1992evolution, yamamoto1994fractal, kobayashi19821, pritchard1992brain, freeman2000spatial, robinson2001prediction, bedard2006does, allegrini2009spontaneous},
whose power spectral density (PSD) follows $1/f^\beta$, can be ubiquitously observed.
Accordingly,
it is imperative to analyze the effects of general noises on information processing besides uncorrelated noise.

In this paper, we theoretically reveal properties of general noise regarding its influence on information processing in a linear RNN that receives input and noise.
We derive an analytical solution of MF to investigate the effects of noise on memory.
Based on the analytical solution, we show two properties. 
First, we derive a simplified representation of the total memory by taking sufficiently large number of nodes to reveal the effect induced by noise correlation.
Second, we introduce a novel way to express MF, and clarify the impact of noise intensity on MF.
We demonstrate these effects of noise using experimental data obtained from the human brain.

\textit{Methods}.---We consider an RNN with a fixed internal weight, which is called an echo state network (ESN)~\cite{jaeger2001echo}.
A discrete-time RNN updates the state as follows:
$\bm{x}_{t+1} = f(\bm{W}\bm{x}_{t} + \bm{w}_1u_{t+1} + \bm{w}_2v_{t+1})$,
where $\bm{x}_{t} \in \mathbb{R}^N$, $u_t\in\mathbb{R}$,
and $v_t\in\mathbb{R}$ denote the state, input, and noise at $t$-th step, respectively,
while $N$ is the number of nodes; $f(\cdot)$ is the activation function and $\bm{W}\in\mathbb{R}^{N\times N}$ is the internal weight matrix; and 
$\bm{w}_1, \bm{w}_2\in\mathbb{R}^N$ are the weight vectors of input and noise, respectively. 
In the present paper, we have adopted the uniformly random input $u_t\in[-1, 1]$ and the linear activation function $f(\cdot)$. 

We evaluate the short-term memory in the RNN using the MF~\cite{jaeger2001short, dambre2012information},
which represents how well the past injected input $u_{t-\tau}$, could be emulated by a linear approximation with the network state:
$\hat{u}_{t-\tau} = \bm{w}_{\text{out}}^{\top} \bm{x}_t$,
where $\tau$ is the delay from the current time.
After determining the readout weight vector $\bm{w}_{\text{out}}$ by minimizing a loss function of the mean squared error (MSE) $1/T\sum_{t=1}^T (\hat{u}_{t-\tau}-u_{t-\tau})^2$,
the MF is obtained as follows:
\begin{equation}
M[u_{t-\tau}] = 1 -
\frac{
\text{min}_{\bm{w}_{\text{out}}} \langle (\hat{u}_{t-\tau}-u_{t-\tau})^2 \rangle
}{\langle u_{t-\tau}^2\rangle} ~(\leq 1),
\label{original MF}
\end{equation}
where $T$ is the sampled time length and 
$ \langle \cdot \rangle$ denotes the time average.
The upper bound $1$ is satisfied when the system has fully memorized the input required to reconstruct the target. 
The sum of the MF with respect to all $\tau$ represents the memory capacity (MC) of the full system, 
$\Msumu =\sum_{\tau=0}^\infty M[u_{t-\tau}]$.
The upper bound of $\Msumu$ is the number of linearly independent time series in the state, which is called rank and ideally $N$.

To demonstrate an applicable range of our results, 
we have utilized not only noise models~\cite{hanggi1993can, fox1988fast, lorenzo1999colored, farlow2006introduction, aizawa1984f, aizawa1989non, akimoto2007new,cryer1986time, langbein2012estimating} but also experimental data of human brain activities that show a $1/f$-like property~\cite{pritchard1992brain, freeman2000spatial, robinson2001prediction, bedard2006does, allegrini2009spontaneous}.
We adopted the EEG signals~\cite{wang2022test} measured from three brain areas:
midline frontal (Fz), vertex (Cz), midline parietal (Pz). 
Signals in brain are expected to have a large amount of information, which include not only memory but irrelevant signals coming from different regions of the brain.
In the current study, we have injected both input and EEG signal to the RNN and regard the EEG as noise. 

\textit{Results}.---In the present study, we derive the MF and MC of RNNs that receive input and noise by direct substitution. 
We begin the derivation from the following MF~\cite{dambre2012information}, which is equivalent to Eq.~(\ref{original MF}):
\begin{align}
M[u_{t-\tau}] &=
\frac{\bm{U}_{\tau}^{\top}\bm{X}(\bm{X}^{\top}\bm{X})^{-1}\bm{X}^{\top}\bm{U}_{\tau}}{\bm{U}_{\tau}^{\top} \bm{U}_{\tau}},\label{formula: MC base}
\end{align}
where $ \bm{X}\in\mathbb{R}^{T\times N} $ is a matrix whose column represents the states time series and $\bm{U}_{\tau}$ is the delayed input series.
In this derivation, we impose two assumptions on the system:
(i) the input $u_t$ and the noise $v_t$ are uncorrelated;
(ii) input and noise share the same weight vector, that is $w = w_1 = w_2$. 
Under the conditions, we derived analytical solutions of MF and MC using a matix $\bm{H}$ and an autocorrelation matrix of $u+v$, $\bm{C}_{uv} = \bm{E} + r \bm{C}_{v}$, where $ r = \langle v^2 \rangle / \langle u^2 \rangle $ is NSR and $\bm{C}_{v}$ is the autocorrelation matrix of $v$. 
$\bm{H}$ is defined as
$\bm{H} = 
\begin{pmatrix}
    \bm{H}_{K-1} & \bm{H}_{K-2} & \cdots & \bm{H}_{0}
\end{pmatrix}$, where 
$\bm{H}_{\tau} = 
\begin{pmatrix}
    {\lambda_1}^{{\tau}} & {\lambda_2}^{{\tau}} & \cdots & {\lambda_N}^{{\tau}}
\end{pmatrix}^{\top}$, 
$K$ is the time length and should be sufficiently large,
and $\lambda_i$ is the $i$-th eigenvalue of $\bm{W}$.
The $(i,j)$ components of $\bm{C}_{v}$ is defined as $(\bm{C}_{v})_{ij}=C(|i-j|)$, $1\leq i,j \leq K$, $i,j \in \mathbb{N}$, 
where $C(\tau)$ is the autocorrelation function of $v$ normalized by the variance.
We call this an analytical solution involving inverse matrix (ASI):
\begin{align}
M[u_{t-\tau}] &=
{\bm{H}_{\tau}}^{\top}
(\bm{H} \bm{C}_{uv} \bm{H}^{\top} )^{-1}
{\bm{H}_{\tau}}, \label{eq:ASI_MF}\\
\Msumu &= \sum_{\tau=0}^{K-1} M[u_{t-\tau}] =
\textrm{tr}\left[
{\bm{H}}^{\top}
( \bm{H} \bm{C}_{uv} \bm{H}^{\top} )^{-1}
{\bm{H}}\right],
\label{eq:ASI_MC}
\end{align}
which are derived in Sec. 1 of Supplementary material.
From ASI, we could confirm that
the MF and MC only depend on the eigenvalues $\lambda_i$ and the autocorrelation of $v$ 
(detailed investigation in Section 2 of Supplementary material).
Based on ASI, we derived the following analytical results.

First, we focus on the MC of a sufficiently large RNN (i.e., the number of nodes and time-series length are infinite $N=K\rightarrow\infty$). 
Under this assumption, we can derive the following formula of $C_{{\rm sum}, u}$ from Eq.~(\ref{eq:ASI_MC}) under the assumption that the rank of system is full:
\begin{align}
\Msumu &= \sum_{i=1}^{N}  \frac{1}{ 1 + r \lambda[\bm{C}_v]_i},
\label{eq:MC_diff_large_func}
\end{align}
where $\lambda[\bm{C}_v]_i$ is the eigenvalues of $\bm{C}_v$ (derived in Section 4 of Supplementary material).
This formula yields two important properties of MC.
(i) The MC becomes independent of $\lambda_i$ and is determined only by $\bm{C}_v$, which are equivalent to the PSD of $v$~\cite{hayes1996statistical, moon2000mathematical}.
The result suggests that, in addition to NSR, the PSD is also crucial in determining the effect of noise, which are numerically verified in Section 4 of Supplementary material.
(ii) The MC of the infinite-dimensional RNN with an arbitrary autocorrelated noise is greater than that with random noise.
We can explain this property by introducing the minimum value of $\Msumu$:
\begin{align}
\Msumu \geq \frac{1}{1+r}N
\label{eq:MC_diff_random_correlated}
\end{align}
which is proven under two conditions that $\textrm{tr}\left[ \bm{C}_v \right]=K$ and the function $\frac{1}{1+x}$ is downward convex
(Theorem 4. 1 in Supplementary material).
The minimum value is the MC with any type of random noise,
meaning that equality holds if the PSD of $v$ is flat and that $\Msumu$ becomes larger if the PSD is not flat.

\begin{figure*}[tbh]
\centering
\includegraphics[width=0.9\linewidth]{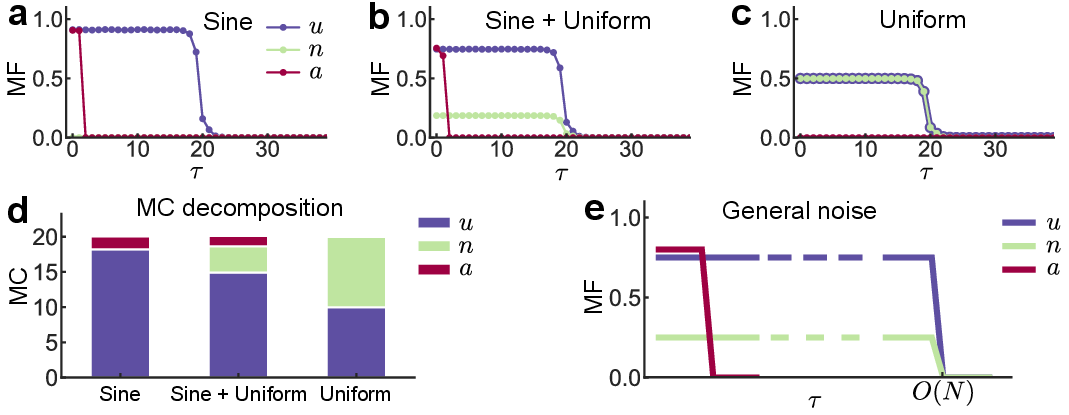}
\caption{\label{fig:MF_decomposed}
The mechanism of memory reduction caused by autocorrelation of noise.
The MF and MC of random input $u$ (purple) and two elements of autocorrelated noise $n$ (green) and $a$ (red) are illustrated. 
To demonstrate the mechanism, the superposition of sinusoidal wave and uniform random noise is used as a noise, whose ratios are $1$ to $0$, $1$ to $1$, and $0$ to $1$ in (a), (b), and (c), respectively.
The input and all superposed noises have the same intensities.
(a)--(c)~show the MFs, whose vertical and horizontal axes are MF and delay, respectively. 
(d)~sums up the upper MFs for each basis $u$, $n$, and $a$, whose horizontal and vertical axes are the label of noise and MC, respectively. 
All analyses were performed with a 20-node system and were repeated for 40 trials to average the MFs and MCs. 
(e)~illustrates the MFs of the infinite-dimensional system with general noise, in which 
the MF of $u_{t-\tau}$ and $n_{t-\tau}$ are much longer than that of $a_{t-\tau}$.
}
\end{figure*}

Second, we focus on MF of both the input and noise.
The preceding result has focused on the memory of input held in RNNs, which also keeps noise as memory; 
however, the MF is conventionally defined with i.i.d. signals and that,
for an autocorrelated noise, has not been derived thus far. 
To address this problem, 
we define the MF from another perspective~\cite{kubota2021unifying}, 
which is the square norm of the coefficient in the state expanded by orthonormal bases and is equivalent to the definition of Eq.~(\ref{original MF}). 
We use an autocorrelated noise model $v_t$ represented by a sum of two factors: i.i.d. noise $\tilde{n}_t$ and time-dependent function $a_t$.
A delayed noise of $v_{t}$ includes delayed series of $\tilde{n}_t$ and orthogonal basis $\tilde{a}_t$ obtained by decomposing $a_t$.
Subsequently, $k$-th delayed noise $v_{t-k}$ can be decomposed into 
$v_{t-k} = c^n_{k}\tilde{n}_{t-k} + \sum_{i=0}^k c^a_{ki} \tilde{a}_{t-i}$,
where the two elements have time averages $\left< \tilde{n}_t\right> = \left< 
\tilde{a}_t\right> = 0$.
The bases $\{\tilde{n}_{t}\}$ are innately defined in the noise and 
$\{\tilde{a}_{t-i}\}$ are defined using the Gram--Schmidt orthogonalization:
\begin{align}
    \hat{a}_{t-k} = a_{t-k} - 
    \sum_{i=0}^{k-1} 
    \left< \tilde{a}_{t-i} {a}_{t-k} \right> \tilde{a}_{t-i},~ 
    \tilde{a}_{t-k} = \frac{\hat{a}_{t-k}}{||\hat{a}_{t-k}||}.
    \label{formula:orthogonalization}
\end{align}
In accordance with the polynomial expansion of state, we have introduced the MFs about the system with an i.i.d. input $u_t$ and an autocorrelated noise $v_t$.
Because the RNN includes only linear terms, the state time series is expanded by three types of time-series bases: $\{u_{t-\tau}\}$, $\{\tilde{n}_{t-\tau}\}$, and $\{\tilde{a}_{t-\tau}\}$. 
The delayed input time-series $\{u_{t-\tau}\}$ organize linearly independent bases because the input at each step is i.i.d.
Additionally, because of the orthogonalization, $\{\tilde{n}_{t-\tau}\}$ and $\{\tilde{a}_{t-\tau}\}$ are also appended to the bases of the orthogonal system. 
As a result, $\{u_{t-\tau}\}$, $\{\tilde{n}_{t-\tau}\}$, and $\{\tilde{a}_{t-\tau}\}$ span the complete orthonormal system for the linear RNN.
Using these bases, we can perfectly expand the state,
and define the MFs on the bases $u_{t-\tau}$, $\tilde{n}_{t-\tau}$, and $\tilde{a}_{t-\tau}$ 
as $M[u_{t-\tau}], M[n_{t-\tau}]$, and $M[a_{t-\tau}]$,
respectively (see Sec. 5 of Supplementary material for derivation).
We can define $\Msumu$ and $\Msumv$ as the MC of each signal, and $\Msum$ as the total MC of the system, which can be computed as
$\Msumu = \sum_{\tau=0}^\infty M[u_{t-\tau}]$,
$\Msumv = \sum_{\tau=0}^\infty M[n_{t-\tau}] + M[a_{t-\tau}]$,
and $\Msum = \Msumu + \Msumv$.
According to the completeness property~\cite{kubota2021unifying}, 
$\Msum= N$ holds because the system depends only on the past input and noise series that span the complete system.

This definition of MF enables us to elucidate the limitations of autocorrelated noise effect. 
Comparing the MF of i.i.d. elements and that of time-dependent function,
their difference can be characterized by the number of bases.
In a system where infinite time has passed, the number of bases $u_{t-\tau}$ would be infinite because the delayed time series would be linearly independent.
For the same reason, the number of bases $\{n_{t-\tau}\}$ is also infinite,
while that of $\tilde{a}_{t-\tau}$ can be both finite (e.g., sinusoidal curve) and infinite (e.g., $1/f$ noise).
Here, we begin with the assumption that the number $N_a$ of bases $\tilde{a}_{t-\tau}$ is finite,
showing that, with a sufficiently large $N$, the sum of MFs about $a_{t-\tau}$ would become $0$:
\begin{align}
\lim_{N\rightarrow\infty}\frac{1}{N}\sum_{\tau=0}^\infty M[a_{t-\tau}] = 0
\label{formula:Csum auto}
\end{align}
which is derived based on the following generating procedure of the base $\tilde{a}_{t-\tau}$.
A base of $\tilde{a}_{t-\tau}$ is newly generated when $\tau$ increments (Eq.~\ref{formula:orthogonalization}). 
Subsequently, the new base is removed if the current input can be expressed by linear combination of existing orthogonal polynomials. This procedure is repeated until the new base does not appear, and we finally obtain the finite number $N_a$ of bases in some cases.
For example, if $a_t = \cos(\omega t)$, the bases are $\cos(\omega t)$ and $\sin(\omega t)$, indicating $N_a=2$. 
This mechanism suggests that RNNs integrate the existing memory of the past inputs and overlapped information of the current input due to autocorrelation.
With a sufficiently large $N$, we obtain
$\sum_{\tau=0}^\infty M[a_{t-\tau}] \leq N_a $,
which produces
Eq.~(\ref{formula:Csum auto})
(MFs of $a$ in Fig.~\ref{fig:MF_decomposed}a, b).
Accordingly, combined with the completeness property, we derive $\lim_{N\rightarrow\infty} 1/N \sum_{\tau=0}^\infty 
\left( M[u_{t-\tau}] + M[n_{t-\tau}] \right) = 1$.
In addition, the MFs of these i.i.d. elements can be characterized by Eq.~(\ref{eq:ASI_MF}),
such that the variances $\langle u^2 \rangle$ and $\langle n^2 \rangle $ determine the ratio between the MFs of i.i.d. elements (see Sec. 5 of Supplementary material for derivation),
indicating that they are just scaled (MFs of $u$ and $n$ in Fig.~\ref{fig:MF_decomposed}a--c).
As the MC of $a_{t-\tau}$ gradually becomes $0$,
increase in $N$ leads to the enhancement of $\Msumu/N$ , 
which suggests that the inhibitory effect of noise becomes smaller.
In the infinite-dimensional RNN, 
$\Msumu/N$ ($\Msumv/N$) will converge to a certain value determined by the ratio of $\langle u^2 \rangle$ and $\langle n^2 \rangle $ (Fig.~\ref{fig:MF_decomposed}e): 
\begin{align}
\lim_{N\rightarrow\infty} \frac{\Msumu}{N} = \frac{\langle u^2 \rangle}{\langle u^2 \rangle+\langle n^2 \rangle}.\label{eq:convergence}
\end{align}
If the noise $v$ is composed only of random components $n_{t-\tau}$,
the ratio between the $\Msumu$ and $\Msumv$ is independent of $N$
and there is no enhancement of $\Msumu$ caused by increasing $N$ (Fig.~\ref{fig:MF_decomposed}c, d). 
In a system with finite $N_a$ (Fig.~\ref{fig:MF_decomposed}d),
we can confirm that, as the proportion of $n_{t-\tau}$ within $v$ decreases,
the disturbance effect of noise becomes smaller, which elucidates the result of Eq.~(\ref{eq:MC_diff_random_correlated}). 
Conversely, if $v$ is composed only of $a_{t-\tau}$, the noise has a little inhibitory effect independent of $r$:
\begin{align}
    \lim_{N\to \infty}\frac{\Msumu}{N}=1 
    ~\left(
    \lim_{N\to \infty}\frac{\Msumv}{N}=0
    \right) .\label{eq:little_inhibitory}
\end{align}

\begin{figure}[t]
    \centering
    \includegraphics[width=1.0\linewidth]{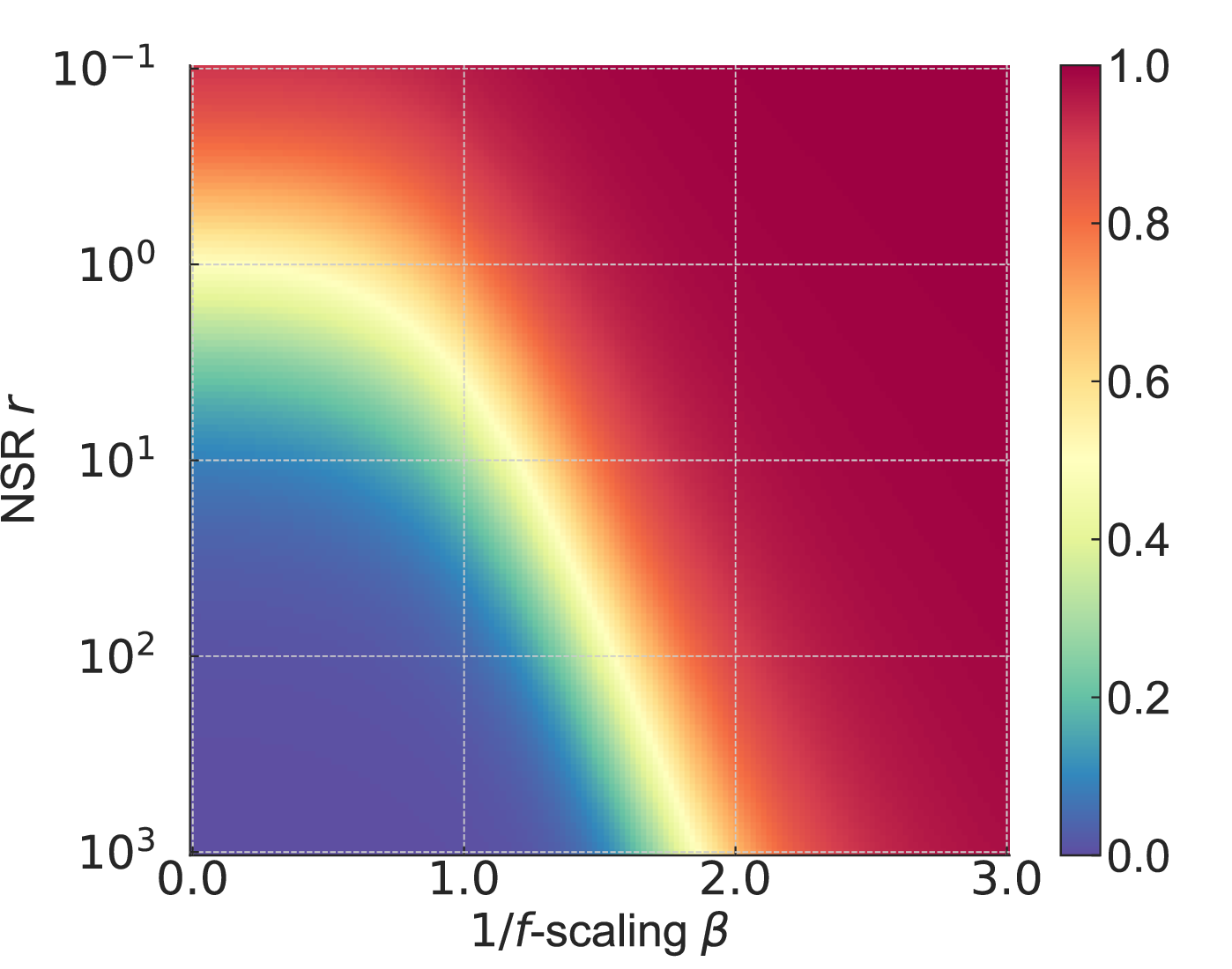}
    \caption{Dependency of normalized MC on $1/f$-scaling of noise and NSR.
    The color indicates $\Msumu/N$, where $N=10^4$.
    The horizontal axis is the $1/f$-scaling $\beta$ of the noise and the vertical axis is NSR $r$. 
    }\label{fig:autocor1}
\end{figure}

Furthermore, even if $N_a$ is infinite, $\Msumv$ could fulfill Eq.~(\ref{eq:little_inhibitory}). 
We proved Eq.~(\ref{eq:little_inhibitory}) in two cases of
(i) $\lim_{N\rightarrow\infty} \Msumv < \infty$ and
(ii) $\lim_{N\rightarrow\infty} \Msumv = \infty$ under conditions.
(i) Under the d'Alembert's test condition of $\lim_{n\rightarrow\infty} \hat{\lambda}[C_v]_{n+1}/\hat{\lambda}[C_v]_{n}<1$, Eq.~(\ref{eq:little_inhibitory}) holds, 
where $\hat{\lambda}[C_v]_{n}$ expresses a sorted version of $\lambda[C_v]_{n}$ in descending order.
(ii) In addition, even if $\lim_{N\rightarrow\infty} \Msumv = \infty$,
$\Msumv$ could hold Eq.~(\ref{eq:little_inhibitory}).
For example, we have proven the case of $1/f^\beta$ noise ($\beta\ge 1$) (see proofs in Sec. 6 of Supplementary material).
Both the examples of PSDs shown here are characterized by the skewed distribution.
We say that a PSD is skewed if the distribution of $\hat{\lambda}[C_v]_{n}$ satisfies Eq.~(\ref{formula:Csum auto}). 
Note that, since PSD represents the magnitude of coefficients of the Fourier series, in which the state is expanded by linearly independent bases of sinusoidal wave,
these results show that the distribution of the magnitude determine the noise effects.
To demonstrate cases in which a very large noise does not hinder information processing in RNN, 
we examined the dependency of normalized MC ($\Msumu/N$) on the parameters $\beta$ and $r$.
Counterintuitively, even if the NSR is large (e.g., $r = 100$), 
the RNN keeps $\Msumu/N \approx 1$ with a large $\beta (>2.0)$. 
Note that a blue region with $\beta\geq 1$ does not satisfy Eq.~(\ref{eq:little_inhibitory}) because $\Msumv/N = o(\log(\log N)/\log N)$
slowly converges to $0$, implying that $N \sim 10^4$ is not sufficiently large.
To mitigate the disturbance, the system requires $N\sim 10^{10^2}$,
implying that the real-world systems (e.g., the number of neurons in brain $N\sim 10^7$~\cite{pakkenberg1997neocortical}) cannot fully hold the MC.

\begin{figure}[t]
    \centering
    \includegraphics[width=1.0\linewidth]{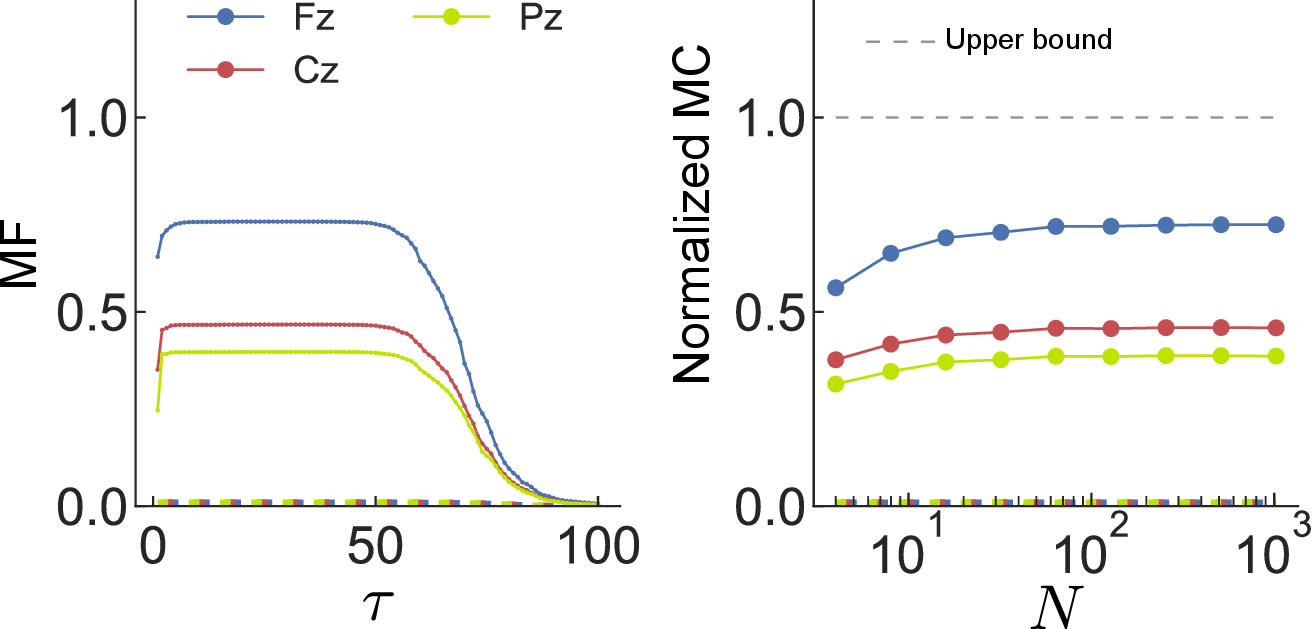}
    \caption{The effect of noise correlation in time series from the real world.
    The left (right) panels show MF (MC), whose horizontal axis is the delay $\tau$ of input (system size $N$). 
    The solid (dotted) lines are calculated with the original (shuffled) EEG series obtained from three electrodes: Fz, Pz, and Cz. 
    The left panel shows $M[u_{t-\tau}]$ in a 128-node RNN.
    The right panel shows $\Msumu$ normalized by $\Msum$, where the upper bound is the rank of system (gray dotted).
    All the plotted lines have been averaged over 40 trials.
    Across all noises, the intensities are the same value ($r=100$).
    }
    \label{fig:autocor2}
\end{figure}

Finally, we numerically verified the effects of noise on the memory using EEG signals from three human brain areas~\cite{wang2022test}.
Comparing the MF and MC of the systems using the autocorrelated noise with those using noises randomly shuffled in time direction (Fig.~\ref{fig:autocor2}),
the former consistently have significantly higher values than the latter. 
Even though the noise intensity is much larger than the input intensity ($r=100$), the MC keeps $\Msumu/\Msum > 0.7$ (``Fz'' in Fig.~\ref{fig:autocor2}, right), where $\Msumu$ is numerically normalized by $\Msum$.
As $N$ increases, $\Msumu/\Msum$ seems to converge to a fixed value,
which may indicate the ratio of $\langle u^2 \rangle$ and $\langle n^2 \rangle$ according to Eq.~(\ref{eq:convergence}).
The three different convergent values may imply that the ratio of random and autocorrelated elements varies among the regions of the brain.
In addition, by incorporating a threshold to determine where the convergence occurs, we can find a sufficiently large $N$ to bring out the maximum MC under the noise.
For example, under the threshold in which $\Msumu/\Msum$ perturbed smaller than $5\times 10^{-3}$, those $N$ of ``Fz'', ``Cz'', and ``Pz''are computed as $40$, $39$, and $37$, respectively.
Overall, we could confirm that the autocorrelated noise in the real-world always has smaller disturbance than random noises and that the memory of the input is retained even with the very strong irrelevant signals.

There are two limitations in our study.
We assumed that the input weights were shared between the input and noise.
This can be interpreted as a setting in which we inject the input containing noise to the RNN and ignore other noises. 
This enables us to analytically evaluate the effects of the initially added noise that entails any intensity and any correlation. 
We need to consider the noises coming from other input pathways in the future.
Next, we have only focused on linear RNNs in the current study.
However, previous researches have confirmed the presence of both linear and nonlinear information processing in real-world systems~\cite{kubota2021unifying, kubota2023temporal},
indicating that nonlinear cases should be investigated in future.

\textit{Summary}.---In the present study,
we used a linear RNN with random input and noise,
including correlated ones, to investigate the effects of noise on memory in general.
We derived an analytical solution of MF and MC dependent only on the autocorrelation of noise and the eigenvalues of internal weight,
which revealed the following three results:
First, in infinite-dimensional systems, 
the MC becomes independent of the internal weight and is determined only by the PSD of noise.
Therefore, our results hold for any type of linear RNNs regardless of whether its internal weights are random or trained. 
By using this solution, we proved that the MC with autocorrelated noise is larger than that with a random noise. 
Second, noise has little inhibitory effects in a sufficiently large system, regardless of its intensity,
if the intensities of linearly independent bases in the noise satisfy a certain condition.
In a general case that the number of base is infinite,
this result is satisfied when the intensities have a skewed distribution that decays at a sufficiently fast rate such as $1/f$ noise.
Actually, this condition is also effective under the case with a finite number of base because its distribution always fulfill the skewed condition such as sinusoidal noise.
Third, we used EEG data to verify the above analytical results.
We demonstrated that, as a form of noise, the EEG series had a small inhibitory effect,
despite its strong intensity.
Moreover, different brain regions showed different ratio of random component and system size to bring out the maximum MC.
From these results, our research has clarified the effects of general noises on information processing,
providing analytical explanation.

\bibliography{main}
\end{document}


\preprint{APS/123-QED}

\title{Supplementary materials: How noise affects memory in linear recurrent networks}



\author{JingChuan Guan}
\email{kan@isi.imi.i.u-tokyo.ac.jp}

\author{Tomoyuki Kubota}

\author{Yasuo Kuniyoshi}

\author{Kohei Nakajima}

\affiliation{%
Intelligent Systems and Informatics Laboratory, Graduate School of Information Science and Technology, The University of Tokyo, 7-3-1 Hongo, Bunkyo-Ku, Tokyo, Japan.
}%



\date{\today}
\maketitle

This supplementary material provides derivations about the formulas and properties introduced in the main text. Also, detailed investigation about what really affects the memory is included in the section 2.
\tableofcontents
\section{The derivation of ASI}
Firstly, we define the linear echo state network. 
Let $u_t\in\mathbb{R}$ be an i.i.d. input and $v_t\in\mathbb{R}$ be a noise. 
The $N$-dimensional state $\bx_t$ is described by 
\begin{align}
    \bx_{t+1} = \bW \bx_t + \bm{w}_1 u_{t+1} + \bm{w}_2 v_{t+1}, 
\end{align}
where $\bW\in\mathbb{R}^{N\times N}$ is the internal weight matrix.
$\bm{w}_1$ and $\bm{w}_2$ $\in\mathbb{R}^N$ represent the input weight vectors of the input and noise, respectively. 

We begin the derivation of the MF from the following formula~\cite{dambre2012information}:
\begin{equation}
M[\bm{X}, \bm{Z}]=\frac{\bm{Z}^{\top}\bm{X}(\bm{X}^{\top}\bm{X})^{-1}\bm{X}^{\top}\bm{Z}}{\bm{Z}^{\top}\bm{Z}},
\end{equation}
where 
$\bm{X}$ is the matrix composed of time-series of $\bm{x}$ which is the state vector of the system, and
$\bm{Z}$ is the target series defined as the delayed time-series of past input.

\subsection{ASI with correlation between input and noise}
We first calculate the state vector $\bm{x}_{t}$:
\begin{align}
    \bm{x}_{t}
    =&\bm{W}^{t}\bm{x}_0
    +\bm{W}^{t-1}(\bm{w}_1 u_1 + \bm{w}_2 v_1)
    +...
    +\bm{W}(\bm{w}_1 u_{t-1} + \bm{w}_2 v_{t-1})
    +(\bm{w}_1 u_{t} + \bm{w}_2 v_{t}) \nonumber\\
    =&\bm{W}^{t}\bm{x}_0+\sum_{i=1}^{t}\bm{W}^{i-1}(\bm{w}_1 u_{t-i+1} + \bm{w}_2 v_{t-i+1})\nonumber\\
    =&\sum_{i=1}^{t}\bm{W}^{i-1}(\bm{w}_1 u_{t-i+1}+\bm{w}_2 v_{t-i+1}),
\end{align}
where initial state is $\boldsymbol{x}_0 = \boldsymbol{0}$.
We simplify $\bm{x}_{t}$ through the eigenvalue decomposition of internal weight matrix
$\bm{W} = \bm{P} \bm{\Sigma} \bm{P}^{\top}$,
where 
$\bm{\Sigma}=
\begin{pmatrix}
\lambda_1                                           \\
     & \lambda_2       &        & \text{\huge{0}}   \\
     &                 & \ddots                     \\
     & \text{\huge{0}} &        & \ddots            \\
     &                 &        &           & \lambda_N
\end{pmatrix}
$, and $\lambda_m$ are the $m$-th eigenvalues of $\bm{W}$.
The largest $\lambda_m$ is defined as the spectral radius $\rho$.
Therefore, $\bm{x}_{t}$ is described by
\begin{align} 
\bm{x}_{t}
=\sum_{i=1}^{t}W^{i-1}(\bm{w}_1 u_{t-i+1}+\bm{w}_2 v_{t-i+1})
=\sum_{i=1}^{t}\bm{P} \bm{\Sigma}^{i-1} \bm{P}^{\top}
(\bm{w}_1 u_{t-i+1}+\bm{w}_2 v_{t-i+1}).
\end{align}
Using new vectors:
$
\bm{P}^{\top}\bm{w}_1=
\begin{pmatrix}
    p_{1\bm{w}_1}'\\
    p_{2\bm{w}_1}'\\
    \vdots\\
    p_{N\bm{w}_1}'
\end{pmatrix},
\bm{P}^{\top}\bm{w}_2=
\begin{pmatrix}
    p_{1\bm{w}_2}'\\
    p_{2\bm{w}_2}'\\
    \vdots\\
    p_{N\bm{w}_2}'
\end{pmatrix}
$, we continue calculations:
\begin{align}
\bm{x}_{t}
=&\sum_{i=1}^{t} \bm{P} \bm{\Sigma}^{i-1} \bm{P}^{\top} \bm{w}_1  u_{t-i+1}
+ \sum_{i=1}^{t} \bm{P} \bm{\Sigma}^{i-1} \bm{P}^{\top} \bm{w}_2  v_{t-i+1}\\
=&\bm{P} \sum_{i=1}^{t}  u_{t-i+1} \bm{\Sigma}^{i-1}
\begin{pmatrix}
    p_{1\bm{w}_1}'\\
    p_{2\bm{w}_1}'\\
    \vdots\\
    p_{N\bm{w}_1}'
\end{pmatrix}
+ \bm{P} \sum_{i=1}^{t}  v_{t-i+1} \bm{\Sigma}^{i-1} 
\begin{pmatrix}
    p_{1\bm{w}_2}'\\
    p_{2\bm{w}_2}'\\
    \vdots\\
    p_{N\bm{w}_2}'
\end{pmatrix}\\
=&
\bm{P}
\sum_{i=1}^{t}
(p_{1\bm{w}_1}'\begin{pmatrix}
    \lambda_1^{i-1}\\
    0\\
    \vdots\\
    0
\end{pmatrix}
+
p_{2\bm{w}_1}'
\begin{pmatrix}
    0\\
    \lambda_2^{i-1}\\
    0\\
    \vdots
\end{pmatrix}
+\hdots+
p_{N\bm{w}_1}'\begin{pmatrix}
    0\\
    \vdots\\
    0\\
    \lambda_N^{i-1}
\end{pmatrix}
)
 u_{t-i+1}
\\
&+
\bm{P}
\sum_{i=1}^{t}
(p_{1\bm{w}_2}'\begin{pmatrix}
    \lambda_1^{i-1}\\
    0\\
    \vdots\\
    0
\end{pmatrix}
+
p_{2\bm{w}_2}'
\begin{pmatrix}
    0\\
    \lambda_2^{i-1}\\
    0\\
    \vdots
\end{pmatrix}
+\hdots+
p_{N\bm{w}_2}'\begin{pmatrix}
    0\\
    \vdots\\
    0\\
    \lambda_N^{i-1}
\end{pmatrix}
)
 v_{t-i+1}\\
=&
\bm{P}
\sum_{i=1}^{t}
\begin{pmatrix}
    p_{1\bm{w}_1}'\lambda_1^{i-1}\\
    p_{2\bm{w}_1}'\lambda_2^{i-1}\\
    \vdots \\
    p_{N\bm{w}_1}'\lambda_N^{i-1}
\end{pmatrix}
 u_{t-i+1}
+
\bm{P}
\sum_{i=1}^{t} 
\begin{pmatrix}
    p_{1\bm{w}_2}'\lambda_1^{i-1}\\
    p_{2\bm{w}_2}'\lambda_2^{i-1}\\
    \vdots \\
    p_{N\bm{w}_2}'\lambda_N^{i-1}
\end{pmatrix}
 v_{t-i+1}.
\end{align}
Defining new vectors:
$
\bm{U}_k=
\begin{pmatrix}
    u_{t-k-(K-1)}\\
    \vdots \\
    u_{t-k-(1)}\\
    u_{t-k-(0)}
\end{pmatrix}
, ~
\bm{V}_k=
\begin{pmatrix}
    v_{t-k-(K-1)}\\
    \vdots \\
    v_{t-k-(1)}\\
    v_{t-k-(0)}
\end{pmatrix}$, ~
$
\bm{\Lambda}_m
=\begin{pmatrix}
    {\lambda_m}^{K-1} \\
    \vdots \\
    {\lambda_m}^{1} \\
    {\lambda_m}^{0}
\end{pmatrix}
$, where $K$ is sufficiently large,
we obtain
\begin{align}
\bm{x}_{t-k}
&=
\bm{P}
\begin{pmatrix}
    p_{1\bm{w}_1}' \bm{\Lambda}_1^{\top} \bm{U}_k\\
    p_{2\bm{w}_1}' \bm{\Lambda}_2^{\top} \bm{U}_k\\
    \vdots \\
    p_{N\bm{w}_1}' \bm{\Lambda}_N^{\top} \bm{U}_k
\end{pmatrix}
+
\bm{P}
\begin{pmatrix}
    p_{1\bm{w}_2}' \bm{\Lambda}_1^{\top} \bm{V}_k\\
    p_{2\bm{w}_2}' \bm{\Lambda}_2^{\top} \bm{V}_k\\
    \vdots \\
    p_{N\bm{w}_2}' \bm{\Lambda}_N^{\top} \bm{V}_k
\end{pmatrix}.
\end{align}
Additionally, we define new matrices $\bm{Q}_1$ and $\bm{Q}_2$, each of which is a matrix that all the $i$-th column vectors of $\bm{P}$ are multiplied by $p_{i\bm{w}_1}'$ and $p_{i\bm{w}_2}'$ respectively.
Finally, we can express $\bm{x}_{t-k}$ as
\begin{align}
\bm{x}_{t-k}=&
\bm{Q}_1
\begin{pmatrix}
    \bm{\Lambda}_1^{\top} \bm{U}_k\\
    \bm{\Lambda}_2^{\top} \bm{U}_k\\
    \vdots\\
    \bm{\Lambda}_N^{\top} \bm{U}_k
\end{pmatrix}
+
\bm{Q}_2
\begin{pmatrix}
    \bm{\Lambda}_1^{\top} \bm{V}_k\\
    \bm{\Lambda}_2^{\top} \bm{V}_k\\
    \vdots\\
    \bm{\Lambda}_N^{\top} \bm{V}_k
\end{pmatrix}.
\end{align}

Next, we assumed the condition that $ \bm{w}_1 = \bm{w}_2$ ($\bm{Q}=\bm{Q}_1=\bm{Q}_2$), and compute $\bm{X}$:
\begin{align}
\bm{X}^{\top} =&
\begin{pmatrix}
\bm{Q}
\begin{pmatrix}
    \bm{\Lambda}_1^{\top} \bm{U}_{K-1}\\
    \bm{\Lambda}_2^{\top} \bm{U}_{K-1}\\
    \vdots\\
    \bm{\Lambda}_N^{\top} \bm{U}_{K-1}
\end{pmatrix}
+
\bm{Q}
\begin{pmatrix}
    \bm{\Lambda}_1^{\top} \bm{V}_{K-1}\\
    \bm{\Lambda}_2^{\top} \bm{V}_{K-1}\\
    \vdots\\
    \bm{\Lambda}_N^{\top} \bm{V}_{K-1}
\end{pmatrix}
& \hdots &
\bm{Q}
\begin{pmatrix}
    \bm{\Lambda}_1^{\top} \bm{U}_0\\
    \bm{\Lambda}_2^{\top} \bm{U}_0\\
    \vdots\\
    \bm{\Lambda}_N^{\top} \bm{U}_0
\end{pmatrix}
+ \bm{Q}
\begin{pmatrix}
    \bm{\Lambda}_1^{\top} \bm{V}_0\\
    \bm{\Lambda}_2^{\top} \bm{V}_0\\
    \vdots\\
    \bm{\Lambda}_N^{\top} \bm{V}_0
\end{pmatrix}
\end{pmatrix}\\
=&
\bm{Q}
\begin{pmatrix}
\begin{pmatrix}
    \bm{\Lambda}_1^{\top} \bm{U}_{K-1}+\bm{\Lambda}_1^{\top} \bm{V}_{K-1}\\
    \bm{\Lambda}_2^{\top} \bm{U}_{K-1}+\bm{\Lambda}_2^{\top} \bm{V}_{K-1}\\
    \vdots\\
    \bm{\Lambda}_N^{\top} \bm{U}_{K-1}+\bm{\Lambda}_N^{\top} \bm{V}_{K-1}
\end{pmatrix}
& \hdots &
\begin{pmatrix}
    \bm{\Lambda}_1^{\top} \bm{U}_0+\bm{\Lambda}_1^{\top} \bm{V}_0\\
    \bm{\Lambda}_2^{\top} \bm{U}_0+\bm{\Lambda}_2^{\top} \bm{V}_0\\
    \vdots\\
    \bm{\Lambda}_N^{\top} \bm{U}_0+\bm{\Lambda}_N^{\top} \bm{V}_0
\end{pmatrix}
\end{pmatrix}.
\end{align}
Defining a new matrix $\bm{G}$ as
\begin{align}
\bm{G}^{\top}
&=
\begin{pmatrix}
\begin{pmatrix}
    \bm{\Lambda}_1^{\top} \bm{U}_{K-1}+\bm{\Lambda}_1^{\top} \bm{V}_{K-1}\\
    \bm{\Lambda}_2^{\top} \bm{U}_{K-1}+\bm{\Lambda}_2^{\top} \bm{V}_{K-1}\\
    \vdots\\
    \bm{\Lambda}_N^{\top} \bm{U}_{K-1}+\bm{\Lambda}_N^{\top} \bm{V}_{K-1}
\end{pmatrix}
& \hdots &
\begin{pmatrix}
    \bm{\Lambda}_1^{\top} \bm{U}_0+\bm{\Lambda}_1^{\top} \bm{V}_0\\
    \bm{\Lambda}_2^{\top} \bm{U}_0+\bm{\Lambda}_2^{\top} \bm{V}_0\\
    \vdots\\
    \bm{\Lambda}_N^{\top} \bm{U}_0+\bm{\Lambda}_N^{\top} \bm{V}_0
\end{pmatrix}
\end{pmatrix}\\
&=
\begin{pmatrix}
    {{\bm{\Lambda}}_1}^{\top}\\
    \vdots\\
    {{\bm{\Lambda}}_N}^{\top}
\end{pmatrix}
\begin{pmatrix}
    \bm{U}'_{K-1} & \cdots & \bm{U}'_{0}
\end{pmatrix},
\end{align}
where $ \bm{U}'_k = \bm{U}_k + \bm{V}_k $,
we obtain
$
{\bm{X}}^{\top} \bm{X} = \bm{Q} \bm{G}^{\top}\bm{G} \bm{Q}^{\top},
\bm{X}^{\top}\bm{Z} = \bm{Q}\bm{G}^{\top}\bm{U}_k
$, and
\begin{align}
M[\bm{X}, \bm{U}_k]
=&
\frac{
\bm{Z}^{\top}\bm{X}
(\bm{X}^{\top} \bm{X})^{-1}
\bm{X}^{\top}\bm{Z}
}{\bm{Z}\bm{Z}^{\top}}\\
=&
\frac{(\bm{Q}\bm{G}^{\top}\bm{U}_k)^{\top}
    (\bm{Q} \bm{G}^{\top}\bm{G} \bm{Q}^{\top})^{-1}
    \bm{Q}\bm{G}^{\top}\bm{U}_k
    }{\bm{U}_k^{\top} \bm{U}_k}\\
=&
\frac{\bm{U}_k^{\top}\bm{G}
    (\bm{G}^{\top}\bm{G} )^{-1}
    \bm{G}^{\top}\bm{U}_k
    }{\bm{U}_k^{\top} \bm{U}_k}.
\end{align}
In this expression of ASI, the correlation between input and noise is still considered. 
Importantly, we can already confirm that, in the system side, only one type of parameter remains, which is the eigenvalues $\lambda_m$ of the internal weight matrix $\bm{W}$.

\subsection{ASI}
We continued to simplify the analytical solution. 
Regarding the calculation of $\bm{G}^{\top}\bm{G}$, 
we used the assumption that there is no correlation between the input and noise. 
Using ${\bm{U}'_k}= \lbrace u'_{t+i-(K-1)-k} \rbrace$, where $i=0, \cdots ,K-1$,
we obtain
\begin{align}
\bm{G}^{\top}&=
\begin{pmatrix}
    {{\bm{\Lambda}}_1}^{\top}\\
    \vdots\\
    {{\bm{\Lambda}}_N}^{\top}
\end{pmatrix}
\begin{pmatrix}
    \bm{U}'_{K-1} & \cdots & \bm{U}'_{0}
\end{pmatrix}\\
\bm{G}^{\top} \bm{G}
&=
\begin{pmatrix}
    {{\bm{\Lambda}}_1}^{\top}\\
    \vdots\\
    {{\bm{\Lambda}}_N}^{\top}
\end{pmatrix}
\begin{pmatrix}
    \bm{U}'_{K-1} & \cdots & \bm{U}'_{0}
\end{pmatrix}
\begin{pmatrix}
    {\bm{U}'_{K-1}}^{\top} \\
    \vdots \\
    {\bm{U}'_{0}}^{\top}
\end{pmatrix}
\begin{pmatrix}
    {{\bm{\Lambda}}_1} & \hdots & {{\bm{\Lambda}}_N}
\end{pmatrix}.
\end{align}
Here, we use the following property:
$
\begin{pmatrix}
    \bm{U}'_{K-1} & \cdots & \bm{U}'_{0}
\end{pmatrix} =
\begin{pmatrix}
    \bm{U}'_{K-1} & \cdots & \bm{U}'_{0}
\end{pmatrix}^{\top}
$ and 
\begin{align}
{\bm{U}'}_k^{\top} {\bm{U}'}_l =
    \begin{cases}
        ( 1 + r ) K \langle u^2 \rangle
        & (k=l) \\
        r C( |k-l| ) K \langle u^2 \rangle
        & (k \neq l)
    \end{cases}
    ,
\end{align}
where $C(\tau)$ is the autocorrelation function normalized by the variance of $v$.
Subsequently, we can calculate as follows:
\begin{align}
\begin{pmatrix}
    \bm{U}'_{K-1} & \cdots & \bm{U}'_{0}
\end{pmatrix}
\begin{pmatrix}
    {\bm{U}'_{K-1}}^{\top} \\
    \vdots \\
    {\bm{U}'_{0}}^{\top}
\end{pmatrix} = K\langle u^2\rangle \bm{C}_{uv},
\end{align}
where, 
$
\bm{C}_{uv} = 
\bm{E} + \frac{\langle v^2 \rangle}{\langle u^2\rangle} \bm{C}_{v}$, and 
$
\bm{C}_{v} =
\begin{pmatrix}
    1 &  C(1) & \cdots &  C(K-1) \\
     C(1) &  1 & \cdots &  C(K-2)\\
    \vdots &   & \ddots & \vdots\\
     C(K-1) &  C(K-2) &\cdots &  1 \\
\end{pmatrix}
$.
Therefore, we obtain
\begin{align}
\bm{G}^{\top} \bm{G} =
K\langle u^2\rangle
\bm{H} \bm{C}_{uv} \bm{H}^{\top},
\end{align}
where $\bm{H}
=
\begin{pmatrix}
    {{\bm{\Lambda}}_1}^{\top}\\
    \vdots\\
    {{\bm{\Lambda}}_N}^{\top}
\end{pmatrix}
=
\begin{pmatrix}
    {\lambda_1}^{K-1} & {\lambda_1}^{K-2} & \cdots & {\lambda_1}^{1} & 1\\
    {\lambda_2}^{K-1} & {\lambda_2}^{K-2} & \cdots & {\lambda_2}^{1} & 1\\
    \vdots\\
    {\lambda_N}^{K-1} & {\lambda_N}^{K-2} & \cdots & {\lambda_N}^{1} & 1\\
\end{pmatrix}
$.

Next, we compute $\bm{G}^{\top}\bm{U}_k$:  
\begin{align}
\bm{G}^{\top} \bm{U}_k
&=
\begin{pmatrix}
\begin{pmatrix}
    \bm{\Lambda}_1^{\top} \bm{U}_{K-1}+\bm{\Lambda}_1^{\top} \bm{V}_{K-1}\\
    \bm{\Lambda}_2^{\top} \bm{U}_{K-1}+\bm{\Lambda}_2^{\top} \bm{V}_{K-1}\\
    \vdots\\
    \bm{\Lambda}_N^{\top} \bm{U}_{K-1}+\bm{\Lambda}_N^{\top} \bm{V}_{K-1}
\end{pmatrix}
& \hdots &
\begin{pmatrix}
    \bm{\Lambda}_1^{\top} \bm{U}_0+\bm{\Lambda}_1^{\top} \bm{V}_0\\
    \bm{\Lambda}_2^{\top} \bm{U}_0+\bm{\Lambda}_2^{\top} \bm{V}_0\\
    \vdots\\
    \bm{\Lambda}_N^{\top} \bm{U}_0+\bm{\Lambda}_N^{\top} \bm{V}_0
\end{pmatrix}
\end{pmatrix} \bm{U}_k .
\end{align}
By focusing on the $l$-th row,
we can simplify as follows:
\begin{align}
&
\begin{pmatrix}
{\bm{\Lambda}_l}^{\top}(\bm{U}_{K-1} +\bm{V}_{K-1})
& \hdots &
{\bm{\Lambda}_l}^{\top}(\bm{U}_0 +\bm{V}_0)
\end{pmatrix}\bm{U}_k
\\
&=
{\bm{\Lambda}_l}^{\top}(
    (\bm{U}_{K-1} +\bm{V}_{K-1})u_{0-k} 
    +\hdots+
    (\bm{U}_0 +\bm{V}_0)u_{t-k}
)\\
&=
{\bm{\Lambda}_l}^{\top}(
    (\bm{U}_{K-1} +\bm{V}_{K-1})u_{t-(K-1)-k} 
    +\hdots+
    (\bm{U}_0 +\bm{V}_0)u_{t-k}
).
\end{align}
The $m$-th row of the right vector is additionally focused on:
\begin{align}
&
(u_{t-(K-1)-(K-m)} + v_{t-(K-1)-(K-m)} )u_{t-(K-1)-k}
+\hdots+
( u_{t-(K-m)}+v_{t-(K-m)} )u_{t-k}\\
&=
\sum_{i=0}^K u_{t-(K-m)-(K-1)+i} ~ u_{t-(K-1)-k+i}\\
&=
\begin{cases}
    \sum_{i=0}^K u_{t-k+1}^2 & (m=K-k) \\
    0 & (m \neq K-k )
\end{cases},
\end{align}
which produces
$
\begin{pmatrix}
{\bm{\Lambda}_l}^{\top}(\bm{U}_{K-1} +\bm{V}_{K-1})
& \hdots &
{\bm{\Lambda}_l}^{\top}(\bm{U}_t +\bm{V}_t)
\end{pmatrix}\bm{U}_k
=
{{\lambda}_l}^{k} \sum_{i=0}u_{t-i-k}^2
=
{\lambda}_l^{k} K\langle u^2\rangle
$. Therefore, we obtain
\begin{align}
\bm{G}^{\top} \bm{U}_k
= K\langle u^2\rangle \bm{H}_k,
\end{align}
where 
$
\bm{H}_k=
\begin{pmatrix}
    {\lambda_1}^{k}\\
    {\lambda_2}^{k}\\
    \vdots\\
    {\lambda_N}^{k}
\end{pmatrix}
$.

Combining these results, we conclude ASI:
\begin{align}
M[\bm{X}, \bm{U}_k]
&=
K\langle u^2\rangle
{\bm{H}_k}^{\top}
\left[
K\langle u^2\rangle \bm{H} \bm{C}_{uv} \bm{H}^{\top}
\right]^{-1}
{\bm{H}_k}K\langle u^2\rangle/K\langle u^2\rangle\\
&=
{\bm{H}_k}^{\top}
\left[ \bm{H} \bm{C}_{uv} \bm{H}^{\top} \right]^{-1}
{\bm{H}_k}.
\\
\Msumu[\bm{X}]
&=
\textrm{tr}\left[
{\bm{H}}^{\top}
( \bm{H} \bm{C}_{uv} \bm{H}^{\top} )^{-1}
{\bm{H}}\right].\label{ASI}
\end{align}

\subsection{ASInc}
We introduce the ASI whose noise has no correlation (ASInc) here.
The correlation matrix $\bm{C}_{uv}$ is expressed as
$
\bm{C}_{uv}=
\langle u^2 \rangle + \langle v^2 \rangle/\langle u^2\rangle \bm{E}
$, where $\bm{E}$ is identity matrix. Therefore we can introduce
\begin{align}
M[\bm{X}, \bm{U}_k]=&
{\bm{H}_k}^{\top}
\left[ \bm{H}
\frac{\langle u^2 \rangle + \langle v^2 \rangle }{\langle u^2\rangle} \bm{E}
\bm{H}^{\top} \right]^{-1}
{\bm{H}_k}\\
=&
\frac{\langle u^2\rangle}{\langle u^2 \rangle + \langle v^2 \rangle}
{\bm{H}_k}^{\top}
( \bm{H} \bm{H}^{\top} )^{-1}
{\bm{H}_k},\\
\Msumu[\bm{X}]
=&
\frac{\langle u^2\rangle}{\langle u^2 \rangle + \langle v^2 \rangle }
\textrm{tr}\left[
{\bm{H}}^{\top}
( \bm{H} \bm{H}^{\top} )^{-1}
{\bm{H}}\right]\\
\end{align}
With a sufficiently large $K$, we can regard
$
\bm{\Lambda}_l^{\top}\bm{\Lambda}_m
=
\sum_{i=0}^{K-1} (\lambda_l\lambda_m)^i
$ as 
$\frac{1}{1-\lambda_l\lambda_m}$, 
which produces
\begin{align}
M[\bm{X}, \bm{U}_k] =&
\frac{\langle u^2\rangle}{\langle u^2 \rangle + \langle v^2 \rangle }
{\bm{H}_k}^{\top}
\begin{pmatrix}
    \frac{1}{1-\lambda_1\lambda_1} & \frac{1}{1-\lambda_1\lambda_2} & \cdots & \frac{1}{1-\lambda_1\lambda_N} \\
    \frac{1}{1-\lambda_1\lambda_2} & \frac{1}{1-\lambda_2\lambda_2} & & \vdots\\
    \vdots & & \ddots\\
    \frac{1}{1-\lambda_1\lambda_N} & \cdots & & \frac{1}{1-\lambda_N\lambda_N}
\end{pmatrix}^{-1}
{\bm{H}_k}. \label{ASInc}
\end{align}
This solution proves that the MF is determined only by the eigenvalues $\lambda_m$ of $\bm{W}$.

\section{Effects of Weight Eigenvalues and Noise Autocorrelation on Memory}\label{sec:Investigation}
In this section, we examine the influence of $\lambda_m$ and $\bm{C}_v$.

\subsection{Eigenvalue dependence}
\begin{figure}[t]
\centering
\includegraphics[width=1.0\linewidth]{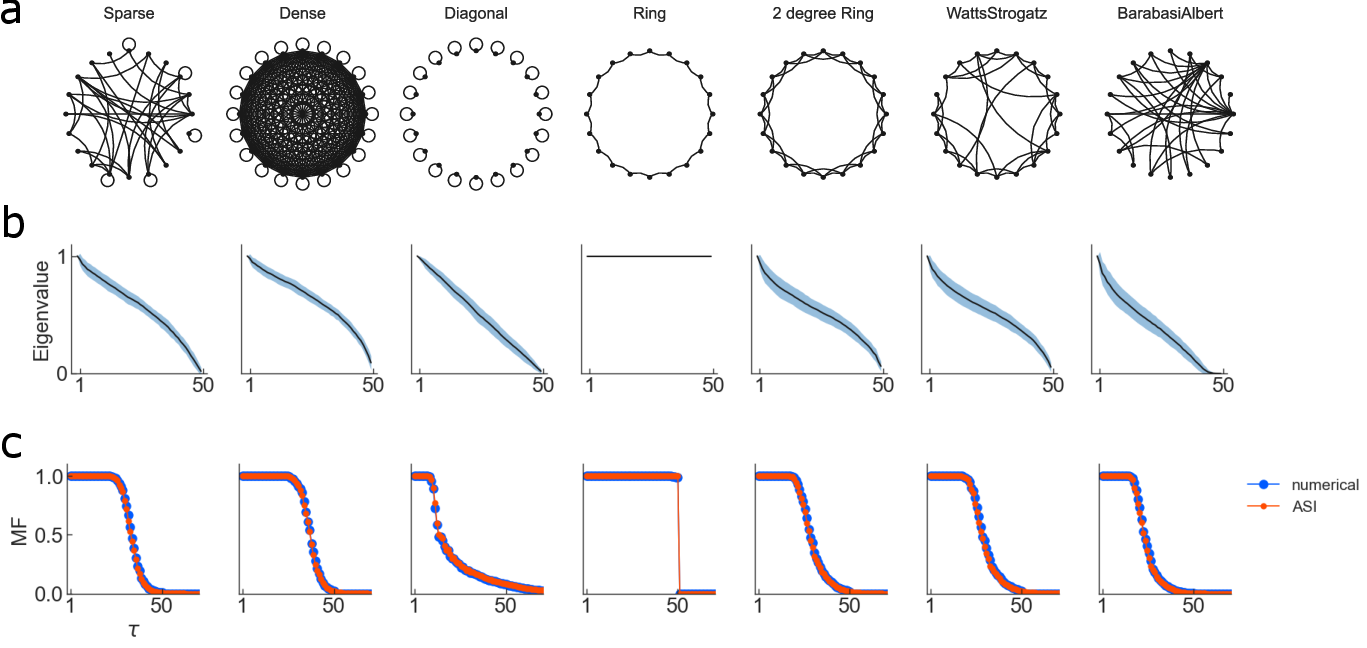}
\includegraphics[width=1.0\linewidth]{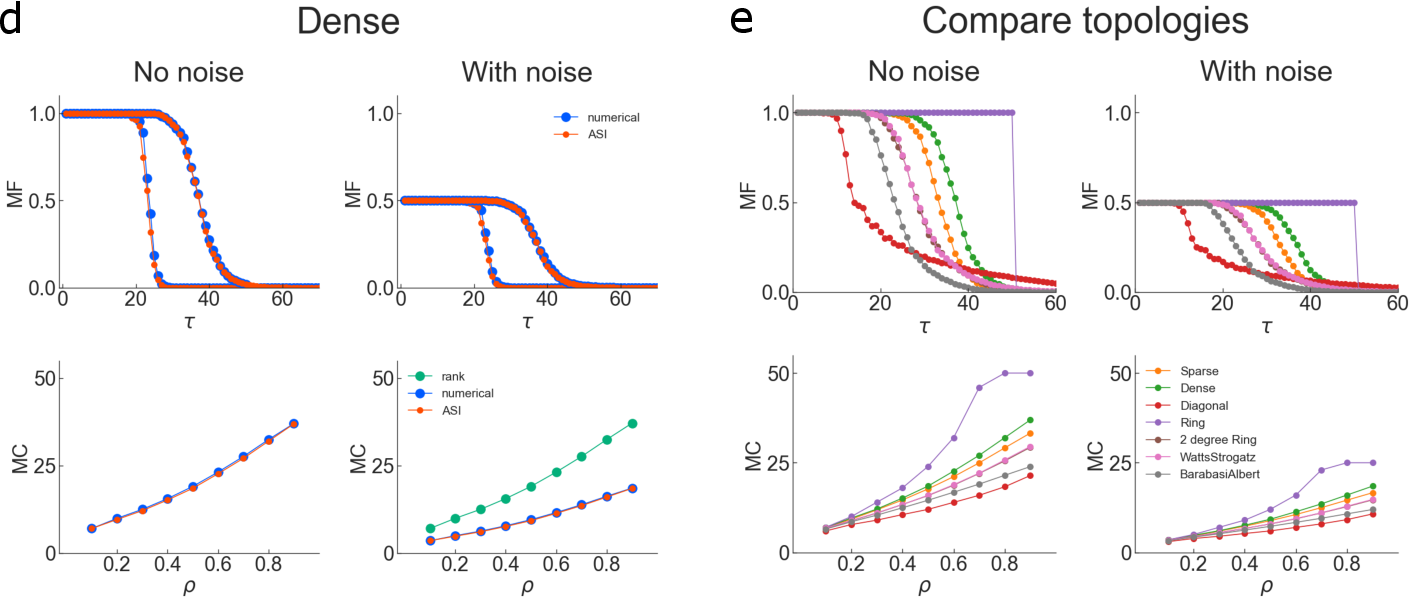}
\caption{\label{fig:topos}
MF and MC with random noise and their eigenvalue dependence.
All the plotted lines about MF and MC are the averaged results of 40 seeds.
(a):~Connectivity of 20 nodes with respect to each topology.
(b):~Eigenvalues of the internal weight matrix with respect to each topology for 50 node systems. The horizontal (vertical) axis is index (eigenvalue). The blue area is the standard deviation.
(c):~The correspondence of MF obtained numerically (numerical, blue) and through ASInc (ASI, red).
For each diagram, the horizontal axis is the delay of input for measuring, and the vertical axis is MF.
(d):~The correspondence of MF and MC focusing on the noise effect.
20 and 50 node systems are shown in the diagrams of MF, and the 20 node system begin to decay faster. 
The case on the Dense topology is calculated numerically (numerical, blue) and from ASInc (ASI, red)
The green line in the panel of MC with noise indicates the rank of the dynamical system which corresponds to the MC without noise.
(e):~Comparisons of topologies on MF and MC obtained through ASInc.
50 node systems are shown for each topology.
In (d) and (e), the above (below) panels show MF (MC) and the left (right) panels display MF and MC without (with) noise. With a noise, its intensity equals the input one.
Regarding the diagrams of MC, the horizontal (vertical) axis is spectral radius $\rho$ (MC).
}
\end{figure}

Previous studies have discussed short-term memory about its network connectivity in terms of topologies that mimic various real-world systems~\cite{dale2021reservoir, ganguli2008memory, kawai2019small}. 
In the present study, we use seven typical topologies (Fig.~\ref{fig:topos}a) referred to as Sparse, Dense, Diagonal, 1 degree Ring (Ring), 
2 degree Ring~\cite{erdds1959random}, 
Watts-Strogatz~\cite{watts1998collective, kawai2019small}, 
and Barabasi-Albert~\cite{barabasi1999emergence}.

The Sparse and Dense topologies were defined with connection probabilities of 0.1 and 1.0, respectively, resulting in varying densities of connections among all nodes. This parameterization allowed us to investigate the effects of changes in connection cost on network behavior.
The Diagonal topology considers only self-loops for each node.
The N degree Ring topologies involve each node forming connections with all nodes closer than the Nth closest nearest nodes. 
Watts-Strogatz is also known as the small world. This model randomly rewires connections of 2 degree Ring with a probability of 0.1 introducing randomness to the original topology.
This model is sometimes used to represent neural networks in biological systems~\cite{watts1998collective, kawai2019small}.
Barabasi-Albert topologies exhibit a power-law distribution of connection density and are considered models for metabolic networks when molecules are represented as nodes~\cite{barabasi2004network}.
The connectivity patterns of these topologies are illustrated in Fig.~\ref{fig:topos}a, while their eigenvalues are shown in Fig.~\ref{fig:topos}b.

It is noteworthy that 
random numbers were employed in generating each topology, and the distribution of these random numbers influences the resulting topologies. 
In this analysis, uniform distribution was used, therefore the results reflect the characteristic of uniform distribution.

We compared the MFs across different topologies to investigate the impact of eigenvalues, including both the analytical solution and numerical solution.
During these experiments, we employ a simple condition that noise does not entail autocorrelation.
The conditions are that the input $u$ is a uniform random number, here considering both cases
(i) without noise (i.e., $v=0$) and 
(ii) with noise ($v$ with a uniform distribution).
The assumption of ASInc hold true under both cases (i) and (ii) because noise has no correlation. Therefore, we only consider ASInc in this experiment.

Initially, we verified that the MF calculated from both ASInc and numerical solution coincided under condition (i) for all introduced topologies.
The results are shown in Fig.~\ref{fig:topos}c, indicating consistency between analytical and numerical solutions across all topologies, indicating that the assumptions used in deriving ASI are valid.
The variance of noise $v$ is 0 under condition (i). 
ASInc does not depend on input properties such as input intensity or distribution (Eq.~\ref{ASInc}). Thus, determining the topology is equivalent to determining the eigenvalues of the internal weight matrix, which, in turn, determines the shape of the MF.

Next, we evaluated the effect of noise under condition (ii).
We generated the time series of the input and noise such that both of them have a same standard deviation, ensuring equal input intensities. 
According to ASInc, the MF and MC of the measured input are proportional to the variance they occupy within the total variance, including noise.
As shown in Fig.~\ref{fig:topos}d, the analytical and numerical solutions agree in both systems with noise and without noise. 
The values of MF on Dense topology with noise are half those without noise, which is also true for all topologies \ref{fig:topos}e.
The half value corresponds to the predicted influence of the variance ratio of the input from ASInc
, which reflects the SNR in the Eq.~(\ref{ASInc}).
In the MC plot, the ranks match the MC without noise.
The ranks in both systems without and with noise correspond to each other, and the MC with noise are half of the ranks. 
Considering that the rank is decided by the number of different values of the matrix $\bm{H}$, we can confirm that the rank is not reduced analytically from Fig.~\ref{fig:topos}b.
However, we can see the ranks do not reach the theoretical values which equal the node number (Fig.~\ref{fig:topos}d below). 
This is the rank reduction problem due to numerical calculations.

Finally, we compared the changes in the MF and MC among topologies (Fig.~\ref{fig:topos}e). 
Considering the time it takes for MF from $\tau=0$ to the point $\tau$ to begin to decay, the order was as follows: 
Ring, Dense, Sparse, Watts-Strogatz, 2 degree Ring, Barabasi-Albert, and Diagonal. 
Considering the order of the time until decay ends, Ring was the longest, while Barabasi-Albert was the shortest, with the other topologies following a similar order as the decay start time. 
However, the difference between Watts-Strogatz and 2 degree Ring was minimal in all cases.
Focusing on MC, the order of topologies was maintained across the entire spectral radius. 
Ranking the topologies in descending order based on MC values aligned with the order of decay start time.
These topologies possess similar eigenvalue profiles, which induced a small difference between their MFs.
This result suggests the importance of investigating the dependence on the eigenvalues of weight matrices rather than topology in the analysis of MF.
As explained in Fig.~\ref{fig:topos}d, the MC without noise equal the rank of the system which corresponds to the rank of matrix $\bm{H} \bm{H}^{\top} $ (Eq.~\ref{ASInc}).
In fact, the rank which matches the MC varies across topologies.
This suggests that one of reasons for diverse memory profiles is induced by the rank reduction problem.

\subsection{Autocorrelated noise}
\begin{figure}[t]
\centering
\includegraphics[width=1.0\linewidth]{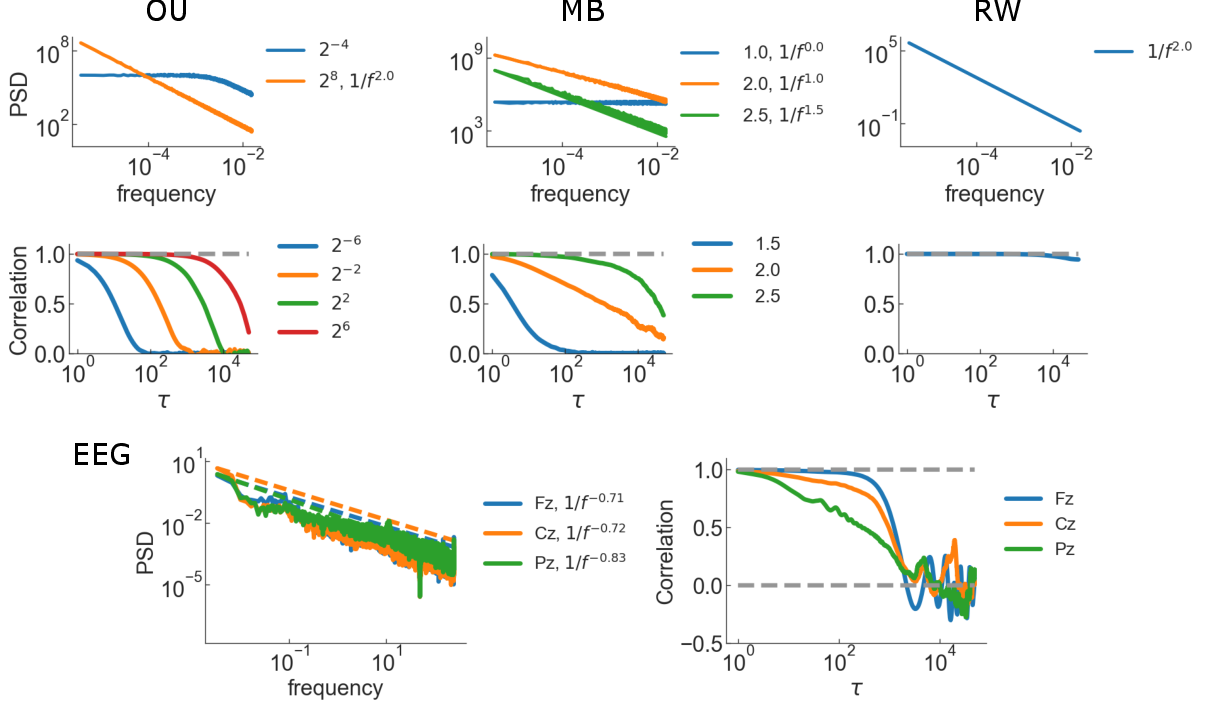}
\caption{\label{fig:autocor1}
PSDs and autocorrelation functions of autocorrelated noise OU, MB, RW, and EEG. 
In PSD plot, the horizontal axis represents frequency $f$, while the vertical axis represents PSD. As for the autocorrelation function plot, the horizontal axis denotes delay $\tau$, while the vertical axis represents autocorrelation.
The labels in the panels of PSD (autocorrelation) are the noise parameters (OU:$\alpha$, MB:$B$, EEG:electrode name).
For the panels of PSD, $1/f$-like properties are also shown. 
}
\end{figure}
\begin{figure}[tbh]
\centering
\includegraphics[width=1.0\linewidth]{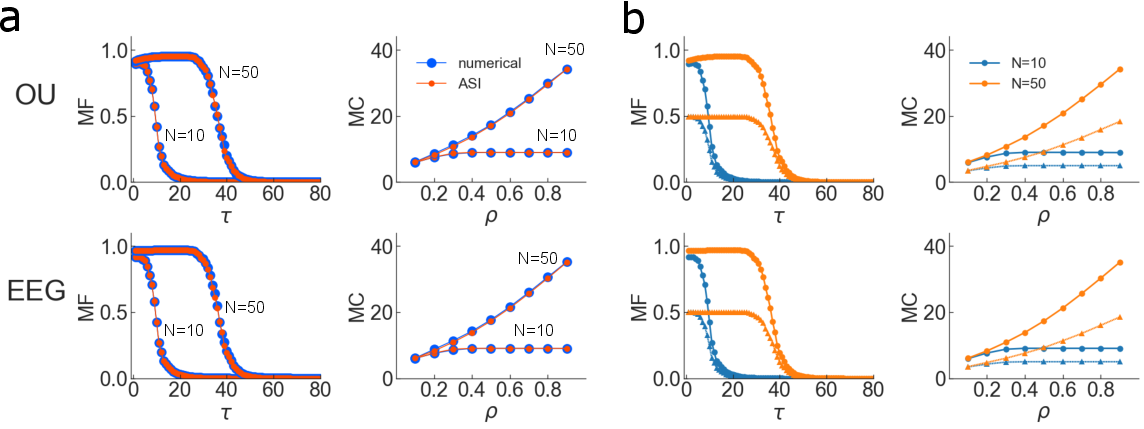}
\includegraphics[width=0.75\linewidth]{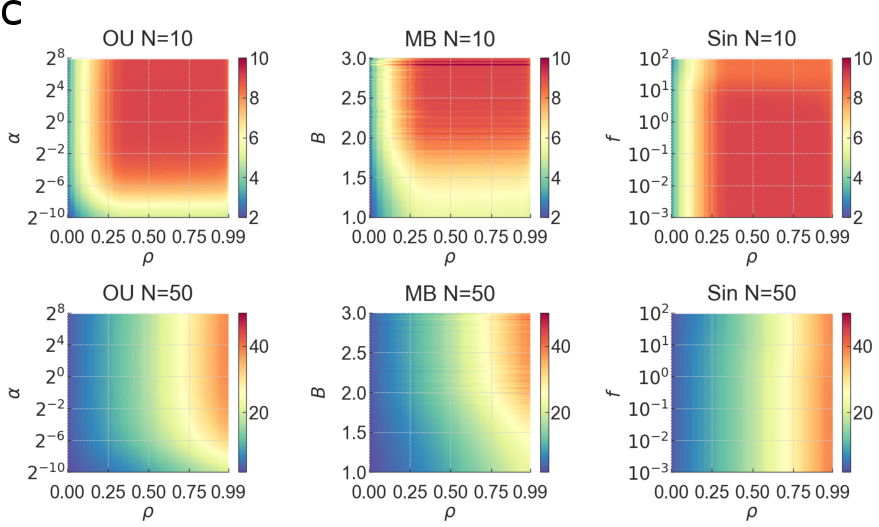}
\caption{\label{fig:autocor2}
MF and MC with autocorrelate noise. All the plotted lines about MF and MC are the averaged results of 40 seeds.
(a):~The correspondence of numerical (blue) and ASI (red).
(b):~Observation on the effect of autocorrelation. The lines with triangle (circular) markers entail the original (shuffled) series.
(c):~Autocorrelation dependency of MC.
In (a) and (b), the left panels show MF, and the horizontal axis is the delay ($\tau$) of input. 
The right panels show MC, and the horizontal axis is spectral radius ($\rho$).
The noise parameter of OU $\alpha$ is $4^{-1}$.
Concerning EEG, the case with the electrode ``Fz'' is plotted.
For both of them, $10$-, $50$-nodes case are shown.
In panel (c), the upper (lower) row depicts the MC with N = 10 (50) nodes.
For each plot, the colors show the value of MC, the horizontal axis is the spectral radius, and the vertical axis is the noise parameters (OU:$\alpha$, MB:$B$, Sin:frequency).
}
\end{figure}

To analyze the effects of correlated noises, we consider three $1/f$-like noise models and a sinusoidal noise. 
According to the Wiener-Khinchin relation, the relationships between the autocorrelation function $C(\tau)$ and the PSD $S(\omega)$ is determined as follows:
\begin{align}
S(k)= \lim_{K\to \infty} \sum_{\tau=0}^{K-1} C(\tau)e^{-i 2\pi \frac{\tau}{K} k},~
C(\tau) = \lim_{K \to \infty} \frac{1}{K}\sum_{k=0}^{K-1} S(k)e^{i 2\pi \frac{k}{K} \tau}.
\end{align}
Therefore, the shape of the PSD decides that of the autocorrelation function, and vice versa.

The first noise model is the Ornstein-Uhlenbeck process (OU)~\cite{hanggi1993can, fox1988fast, lorenzo1999colored,farlow2006introduction},
whose autocorrelation function shows exponential decay.
The time series $v_t$ can be obtained through the following differential equation:
\begin{equation}
    \frac{dv}{dt} = -\frac{1}{\alpha} v + \frac{1}{\alpha} Y(t),
\end{equation}
where $Y$ is a stochastic process and $\alpha$ is a parameter determining the degree of its time correlation.
As $\alpha$ increases, the degree of autocorrelation in the time series becomes larger.
When $\alpha$ is quite large, this time series can be regarded as a random walk.
To generate discrete time series, we employ the following Euler method:
$ v_{t+1} = \left(1-\alpha' \right)v_t + \alpha' Y(t) $,
where $\alpha'=\Delta t/\alpha$ and $0<\alpha'<1$.
In addition, \(C(\tau)\) is analytically derivable as
$ C(\tau) =
\frac{ \alpha' \left(1-\alpha'\right) ^{\tau} }{ 2-\alpha' }
{\rm Var}(Y) $.
The PSD shape indicates Lorentzian shape, and with a sufficiently large $\alpha$, $1/f^2$ property.

The second model is the modified-Bernoulli map (MB)~\cite{aizawa1984f, aizawa1989non, akimoto2007new}, which is defined as follows: 
\begin{equation}
v_{t+1}=\begin{cases}
v_{t}+2^{B-1}(1-2\epsilon) {v_{t}}^B + \epsilon & \text{for $v_{t} \in[0,1/2]$}, \\
v_{t}-2^{B-1}(1-2\epsilon) (1-v_{t})^B - \epsilon & \text{for $v_{t} \in[1/2,1]$},
\end{cases}
\end{equation}
where the control parameter is $B$.
The parameter $\epsilon$ is a quite small constant value.
The series exhibits white noise when $B=1$, $1/f$ property when $B=2$, and $1/f^{B-1}$ for $B\geq2$.

The third noise model is random walk (RW)~\cite{cryer1986time, langbein2012estimating}, 
whose PSD exhibits $1/f^2$. 
The noise $v_t$ is defined as follows:
$Y_{t}(d)$ is a random variable that follows the probability distribution $d$.
\begin{equation}
    v_{t+1}=v_{t}+Y_{t+1}(d).
\end{equation}
The autocorrelation function \(C(\tau)\) is analytically derivable. 
From the generation method, after $t$ time steps, $v_{t}$ follows the 
$t$-th i.i.d. 
The variance of $Y_{t}(d)$ is $\sigma_{Y}^2$, the variance of $v_{t}$ is $\sigma_{t}^2$, so $\sigma_t^2 = t\sigma_{Y}^2$ holds. 
We calculate the correlation between this time series and the time series delayed by $\tau$ time steps.
In the calculation, we assume that $t$ is sufficiently long compared with $\tau$, and we use $t+\tau$ as the time length. 
This yields $ C(\tau) = (t-\tau)\sigma_{Y}^2 $. 
It is evident that $C(\tau)$ depends on the time length, which indicates that the noise is nonstationary.
However, after a sufficiently long time has elapsed, $\tau$ becomes significantly smaller than $t$, and the autocorrelation can be considered a constant function while the duration for measuring the autocorrelation is enough smaller than $t$. 

The last model is the sinusoidal noise (Sin), which shows the typical periodic signal. 
The PSD of a Sin exhibits nonzero values only at a single frequency. The time series $v_t$ is defined as follows using amplitude $A$, frequency $f$, and phase $\phi$:
\begin{align}
    v_t = A \sin (2\pi f t + \phi).
\end{align}
The autocorrelation function is described by
$C(\tau) = \frac{A^2}{2} \sin(2\pi f \tau)$.
We assigned $f$ as the control parameter.

The input intensities of the input and noise are defined as equal.
The PSD and autocorrelation functions for certain parameters are displayed in Fig.~\ref{fig:autocor1} for the OU, MB, RW, and EEG, several electrodes were used. 
Though non-stationary time series were involved, as a sufficiently long period is taken into account, the numerical value obtained from our analyses converge statistically to an unique one.

First, we confirmed ASI matched the numerical solution.
According to Fig.~\ref{fig:autocor2}a, the ASI and the numerical solutions matched with respect to all of the noise type and the number of nodes. 
Consistent with previous studies, we observed that the decay of MF approximately matches the number of nodes and that the dependence of MC on the spectral radius is small at low node numbers but increases as the number of nodes increases.
Based on this result, ASI was employed in the following analyses.

Next, we compared the MF and MC of systems with an autocorrelated noise and a shuffled noise (Fig.~\ref{fig:autocor2}b).
For OU, 
we varied the number of nodes, while for the EEG time series, we varied the electrode positions.
The systems with autocorrelated noises consistently have significantly higher values compared with the systems with shuffled noise. As for MF, there is no change in the beginning and end of decay, indicating a characteristic behavior where MF at each $\tau$ is scaled.

To examine the impact of changing the degree of autocorrelation, we manipulated the noise parameters controlling autocorrelation in OU and MB and the frequency in Sin 
(see Fig.~\ref{fig:autocor2}c). 
Because we have seen changing spectral radius affects the effects of autocorrelation in the previous result, we also incorporated the spectral radius as a variable of MC.
The results revealed MC is increased as the degree of autocorrelation increases. 
Across all noise models, in low-dimensional systems of around 10 nodes, the MC reached the node number, which is upper limit, but it did not reach the upper limit at 50 nodes due to the rank reduction problem.
We can see MC improved as the spectral radius increased.
Especially in systems with 10 nodes, MC reached its upper limit when exceeding 0.3.
For OU, the increase in MC stops when the noise parameters range from $\alpha = 2^{-6}$ or higher. Similarly, for MB, increasing the noise parameter $B$ leads to the maximum MC value when the noise parameters range from $B = 2.0$ or higher. 
It is noteworthy that, for both types of noise, the noise parameter at which the increase stops is independent of the spectral radius.
Sin does not show dependence on the frequency of the noise parameter (explained in Sec. 3).
Because the increase in noise parameter implies an increase in the magnitude of noise autocorrelation for OU and MB, 
the increase in the magnitude of noise autocorrelation means the enhancement of MC is prompted further.

\section{Correlation matrix}
Correlation matrix $\bm{C}_v$ can be Cholesky factorized:
$ \bm{C}_v =
\begin{pmatrix}
    \bm{V}_{K-1} & \bm{V}_{K-2} & \hdots & \bm{V}_0
\end{pmatrix}^{\top}
\begin{pmatrix}
    \bm{V}_{K-1} & \bm{V}_{K-2} & \hdots & \bm{V}_0
\end{pmatrix}
$,
which means that $\bm{C}_v$ is semi-positive definite, and the eigenvalues are all positive. 
$\bm{C}_v$ holds the following property:
\begin{align}
    \operatorname{tr}(\bm{C}_v)=\sum_{i=1}^{K}\lambda[\bm{C}_v]_i = K,
    \label{trace-Cv_sum-PSD}
\end{align}
where $\lambda[\bm{C}_v]_i$ is the eigenvalues of $\bm{C}_v$.
Especially for sinusoidal curve,
we could diagonalize $\bm{C}_v$ and calculate $\lambda[\bm{C}_v]_i$.
\begin{align}
\bm{C}_v
=&
\bm{F}^{-1}
\begin{pmatrix}
    \ddots \\
    & 0 \\
    &        &\frac{K}{2}    &       &        &     & \text{\huge{0}}\\
    &        &               & 0    \\
    &        &               &       & \ddots  \\
    &        &               &       &        & 0 \\
    &        &\text{\huge{0}}&       &        &     & \frac{K}{2}\\
    &        &               &       &        &     &     & 0 \\
    &        &               &       &        &     &     &      & \ddots
\end{pmatrix}
\bm{F},
\label{eq:eigens_Sin}
\end{align}
where $\bm{F}$ is a discrete fourier transform matrix.
The eigenvalue of $\bm{C}_v$ has two $\frac{K}{2}$ in $f$-th and $(K-f)$-th elements, others are all 0.
This proposes that, in general, $\lambda[\bm{C}_v]_i$ is equal to the PSD of the noise which has been revealed in previous studies~\cite{hayes1996statistical, moon2000mathematical}.

To show the effect of changing PSDs, we have adopted several types of PSD model.
``Some peaks" indicates a PSD with two nonzero spikes that depict the PSD of the sinusoidal curve.
The linearly decreasing PSD (``linear") is defined by $\lambda[\bm{C}_v]_{n} = \gamma (n - N_a)$, where $N_a$ represents the number of nonzero values within the PSD.
The exponentially decreasing shapes are represented by ``$10^{- n}$" and ``$10^{-10n}$" defined by $\lambda[\bm{C}_v]_{n}=\gamma 10^{-Bn}$, where $B=1$ and $10$, respectively.
Similar to the linear case, the number $N_a$ of nonzero values is defined.
The $\gamma$ is the coefficient to normalize $\lambda[\bm{C}_v]_{n}$ according to the Eq.~(\ref{trace-Cv_sum-PSD})
The Lorentzian type indicates the PSD defined by $\lambda[\bm{C}_v]_{n}=1/(1+(n/N_w)^2)$.
We have adopted $N_a=N_w=5 \times 10^3$.

\section{Analytical solution calculated from the difference of MC}
\label{sec:Analytical solution delta MC}
We assessed the differences of MC between two systems: one with an autocorrelated noise and another with a random noise. 
Both have the same noise intensity. 
We define the MCs of the two systems: 
$\Msumu^{\rm{ac} }$ and $\Msumu^{\rm{iid} }$, 
which are the $\Msumu$ with an autocorrelated and an i.i.d. noise, 
respectively:
\begin{align}
\Msumu^{\rm{ac} }
&=\textrm{tr}\left[
    \bm{H}^{\top} (\bm{H}\bm{H}^{\top}+\bm{H}(rC_{v}) \bm{H}^{\top} )^{-1} \bm{H} 
    \right] \label{Csum ac} \\
\Msumu^{\rm{iid} }
&=\textrm{tr}\left[
    \bm{H}^{\top} ( (1+r) \bm{H}\bm{H}^{\top} )^{-1} \bm{H}
    \right]. \label{Csum iid}
\end{align}
The difference $\Delta \Msumu$ between $\Msumu^{\rm{ac} }$ and $\Msumu^{\rm{iid} }$ is described by 
$\Delta \Msumu = \Msumu^{\rm{ac} } - \Msumu^{\rm{iid} }$.
We can introduce $\Delta \Msumu$ using Woodbury matrix identity:
\begin{align}
(\bm{A}+\bm{U}\bm{C}\bm{V})^{-1} = \bm{A}^{-1} - \bm{A}^{-1}\bm{U}(\bm{C}^{-1} + \bm{V}\bm{A}^{-1}\bm{U})^{-1}\bm{V}\bm{A}^{-1},
\end{align}
where $\bm{A}=(1+r)\bm{H}\bm{H}^{\top}, \bm{U}=\bm{H}, \bm{C}=r(\bm{C}_{v}-\bm{E}), \bm{V}=\bm{H}^{\top}$,
and we define $r = \langle v^2 \rangle/\langle u^2\rangle$
as noise-to-signal ratio (NSR). 
\begin{align}
&\therefore
( (1+r) \bm{H}\bm{H}^{\top} +  \bm{H} (r\bm{C}_{v} - r\bm{E}) \bm{H}^{\top} )^{-1}
\\
&=
\frac{1}{1+r}(\bm{H}\bm{H}^{\top} )^{-1}
-
\frac{1}{1+r}(\bm{H}\bm{H}^{\top} )^{-1} 
\bm{H}
(
    \frac{1}{r}(\bm{C}_{v} - \bm{E})^{-1}
    +
    \frac{1}{1+r}\bm{H}^{\top} (\bm{H}\bm{H}^{\top} )^{-1} \bm{H}
)^{-1}
\bm{H}^{\top} \frac{1}{1+r} (\bm{H}\bm{H}^{\top} )^{-1},
\end{align}
Defining
$\bm{J}=\bm{H}^{\top} (\bm{H} \bm{H}^{\top} )^{-1} \bm{H}$
and
$\bm{J}'=\bm{H}^{\top} (\bm{H}(1+r)\bm{H}^{\top} )^{-1} \bm{H}=\frac{1}{1+r}\bm{J}$,
we obtain
\begin{align}
&
\bm{H}^{\top} (\bm{H}\bm{H}^{\top}+\bm{H}(r\bm{C}_{v}) \bm{H}^{\top} )^{-1} \bm{H}
-
\bm{H}^{\top} ((1+r)\bm{H}\bm{H}^{\top} )^{-1} \bm{H}
\\
&=
-\bm{H}^{\top} ((1+r)\bm{H}\bm{H}^{\top} )^{-1} \bm{H}((r\bm{C}_{v}-r\bm{E})^{-1}+\bm{H}^{\top} ((1+r)\bm{H}\bm{H}^{\top} )^{-1} \bm{H})^{-1} \bm{H}^{\top}((1+r)\bm{H}\bm{H}^{\top} )^{-1} \bm{H}
\\
&=
-\bm{J}' ( (r\bm{C}_{v} - r\bm{E})^{-1}+\bm{J}')^{-1} \bm{J}'.
\end{align}

\begin{align}
\therefore
\Delta \Msumu=\Msumu^{ac}-\Msumu^{iid}
&=
\textrm{tr}\left[- \bm{J}'( \frac{1}{r}(\bm{C}_{v}-\bm{E})^{-1}+\bm{J}')^{-1} \bm{J}'\right]
\\
&=
\textrm{tr}\left[- ( \frac{1}{r}(\bm{C}_{v}-\bm{E})^{-1}+\bm{J}')^{-1} \bm{J}'\right]
\frac{1}{1+r}
\\
&=
- \frac{1}{1+r} \textrm{tr}\left[( \frac{1}{r}\bm{E}+(\bm{C}_{v}-\bm{E})\bm{J}')^{-1} (\bm{C}_{v}-\bm{E})\bm{J}'\right]  ~~(\because (\bm{A}\bm{B})^{-1}=\bm{B}^{-1} \bm{A}^{-1} )
\\
&=
- \frac{1}{1+r} \textrm{tr}\left[ \bm{E} - \frac{1}{r}( \frac{1}{r}\bm{E}+(\bm{C}_{v}-\bm{E})\bm{J}')^{-1} \right]
\\
&=
\frac{1}{1+r} \textrm{tr}\left[ -\bm{E} + ( \bm{E} - r(\bm{E} - \bm{C}_{v})\bm{J}' )^{-1} \right].
\end{align}
Here, we define
$\operatorname{diag}(\bm{\lambda})$ as the diagonal matrix which aligns the elements of a vector $\bm{\lambda}$ in diagonal elements.
\begin{align}
\therefore
\Delta \Msumu
&=
\frac{1}{1+r} \textrm{tr}\left[ -\bm{E} + ( \bm{E} - \frac{r}{1+r}(\bm{E} - \bm{C}_{v}) \bm{J} )^{-1} \right]
\\
&=
\frac{1}{1+r} \textrm{tr}\left[ -\bm{E} + ( \bm{E} - \frac{r}{1+r}\times \bm{P}_{(\bm{E}-C)\bm{J}}^{\top}~\operatorname{diag}(\alpha_i) ~\bm{P}_{(\bm{E}-C)\bm{J}} )^{-1} \right]
\\
&=
\frac{1}{1+r} \textrm{tr}\left[ -\bm{E} +
    ( \bm{E} - \frac{r}{1+r}\times \operatorname{diag}(\alpha_i))^{-1}
\right]\\
&=
\frac{1}{1+r} 
\sum_{i=1}^{K}
\frac{r\alpha_i}{1+r- r\alpha_i},\label{ASD}
\end{align}
where $ \alpha_i ,~ (i=1,...,K)$ are
the eigenvalues of matrix $ (\bm{E}-\bm{C}_{v}) \bm{J} $ 
and $\bm{P}_{(\bm{E}-C)\bm{J}}$ is the matrix for diagonalizing $(\bm{E}-\bm{C}_v)\bm{J}$.
We call Eq.~(\ref{ASD}) as an analytical solution calculated from the difference of MC (ASD).

Here we explain the result obtained in Fig.~\ref{fig:autocor2}C, showing that the MC are independent of frequency $f$ of the sinusoidal noise.
Combined the correlation matrix (Eq.~\ref{eq:eigens_Sin}) and ASD (Eq.~\ref{ASD}),
because, regarding the noise, the MC is determined only by $\lambda[\bm{C}_v]_i$ which are independent of $f$.

Under the assumption that $N$ is sufficiently large,
we derive $\Msumu$ using ASD.
Because $\alpha_i = 1-\lambda[\bm{C}_v]_i$, we obtain
\begin{align}
\Msumu
&=
\frac{1}{1+r} \textrm{tr}\left[ \bm{I} \right] + \Delta \Msumu \\
&=
\frac{1}{1+r} ( N +
    \sum_{i=1}^{K} 
    \frac{r(1-\lambda[\bm{C}_v]_i)}{ 1 + r - r(1-\lambda[\bm{C}_v]_i)}
    )\\
&=
\frac{1}{1+r} ( N +
    r \sum_{i=1}^{N} 
    \frac{ 1 - \lambda[\bm{C}_v]_i }{ 1 + r \lambda[\bm{C}_v]_i}
    )\\
&=
\sum_{i=1}^{N}
\frac{1}{ 1 + r \lambda[\bm{C}_v]_i}.\label{MC_diff_large_func}
\end{align}

\begin{figure}[t]
    \centering
    \includegraphics[width=0.7\linewidth]{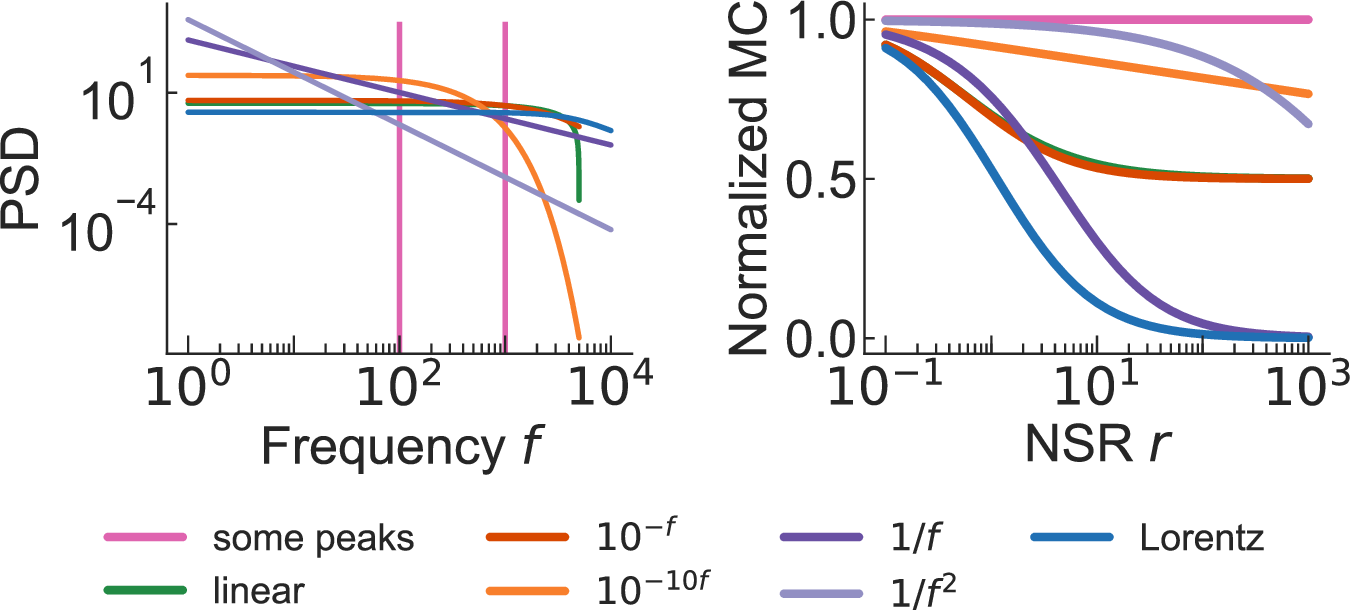}
    \caption{Correspondence between PSD and normalized MC.
    The left panel shows PSD, whose horizontal axis is frequency $f$.
    The right panel shows the normalized MC ($\Msumu/N$), whose horizontal axis is NSR $r$.
    and the color indicates different types of PSD.
    Note that we employ $N=10^4$.
    The PSDs are linear,  exponential ($10^{- \alpha f}$), $1/f^\beta$, and Lorentzian ($1/(1+(f/N_w)^2)$), where $\alpha$ ($=1, 10$), $\beta$ ($=1, 2$), and $N_w$ ($=5\times10^3)$ are fixed values.
    The lacked points indicate $0$ values.
    }\label{fig:MC_PSDs}
\end{figure}

To demonstrate that the PSD $\lambda[\bm{C}_v]_i$ and NSR $r$ of $v$ changes the MC,
we numerically calculated Eq.~(\ref{MC_diff_large_func}) using several types of noise 
(some peaks, linear, exponential, $1/f$-like, and Lorentzian)
, as shown in Fig.~\ref{fig:MC_PSDs}. 
The normalized MC ($\Msumu/N$) shows various disturbance effects dependent on their shapes of PSD.

\begin{thm}
Minimum value of $\Msumu$:
\begin{align}
    \Msumu \geq \frac{1}{1+r}N.
\end{align}
\end{thm}

We consider the minimum problem of $y$:
\begin{gather}
y=\sum_{i=1}^K f(x_i) ~ (f(x):=\frac{1}{1+rx}),
\\
\text{Constraints: }
x_i > 0 ~ (i \in \mathbb{N}, 1\leq i \leq K), ~~ 
-K+\sum_{i=1}^K x_i=0,
\end{gather}
where the restriction of $x_i$ is imposed by Eq.~(\ref{trace-Cv_sum-PSD}).
We prove this theorem by proving that $\Msumu$ holds the minimum value when
\begin{align}
    x_i =1 , ~~  (i \in \mathbb{N}, 1\leq i \leq K). \label{condition of min y}
\end{align}
To prove this, we define the maximum difference $\epsilon = \max_{i,j}{(|x_i - x_j|)}$ among the set $\lbrace (x_i, x_j)~|~i,j \in \mathbb{N}, 1\leq i,j \leq K\rbrace$.
We prove Eq.~(\ref{condition of min y}) by confirming $\epsilon=0$ when y is minimized.

\begin{proof}
Firstly, we prepare a tool for the proof.
We define an operation that works on 2 values $x_i$ and $x_j$.
These two values are chosen from the set 
$\lbrace x_i~|~i \in \mathbb{N}, 1\leq i \leq K\rbrace$.
In order to find the smaller $y$, we update $x_i$ and $x_j$, 
keeping the sum of $x_i$ and $x_j$ are constant, the value $C$.
\begin{align}
x_i + x_j &= C, ~(i,j \in \mathbb{N}, 1\leq i < j \leq K)\\
\hat{y}
    &=f(x_i) + f(x_j)\\
    &=f(x_i) + f(C-x_i),\\
\frac{d}{dx_i}\hat{y}&= \frac{d}{dx_i}f(x_i) - \frac{d}{dx_i}f(C-x_i)=0,
\end{align}
the value of $\hat{y}$ is minimized when $x_i=x_j=\frac{C}{2}$.
We refer this operation as "min$\hat{y}$" in the below.

Next, we define a set sequence $\lbrace A_k~|~A_k = (x_1^k, x_2^k, \cdots, x_K^k)\rbrace_{k \in \mathbb{N}}$, where as $k$ increments, $y$ decreases.
To minimize $y$, we find the map $G: A_{k+1} = G(A_k)$.
$G$ is defined as follows.
We sort $(x_1^k, x_2^k, \cdots, x_K^k)$ in descending order, then create $({x_1^k}', {x_2^k}', \cdots, {x_K^k}')$ .
We apply min$\hat{y}$ to all 
$({x_i^k}', {x_{K-i}^k}') ~(i \in \mathbb{N}, ~1\leq i \leq K/2)$ and create new set in the below,
\begin{align}
A_{k+1} = 
\begin{cases}
(\frac{{x_1^k}'+{x_K^k}'}{2}, \frac{{x_1^k}'+{x_K^k}'}{2}, \frac{{x_2^k}'+{x_{K-1}^k}'}{2}, \frac{{x_2^k}'+{x_{K-1}^k}'}{2}, \cdots, \frac{{x_{K/2-1}^k}' + {x_{K/2+1}^k}'}{2})
& (K/2 \in \mathbb{N})\\
(\frac{{x_1^k}'+{x_K^k}'}{2}, \frac{{x_1^k}'+{x_K^k}'}{2}, \frac{{x_2^k}'+{x_{K-1}^k}'}{2}, \cdots, \frac{{x_{(K-1)/2}^k}' + {x_{(K+3)/2}^k}'}{2}, {x_{(K+1)/2}^k}')
& ((K-1)/2 \in \mathbb{N})
\end{cases}.
\end{align}
In addition, we define the maximum difference between 2 values among $A_k$ as $\epsilon_k = x_i^k-x_j^k$ to prove the following formula using mathematical induction: 
\begin{align}
    \epsilon_n \leq \epsilon_1/2^{n-1} ~(n \in \mathbb{N}).
\end{align}
We show this is true for $m=1$,
\begin{align}
    \epsilon_1 \leq \epsilon_1/2^{1-1}.
\end{align}
Assuming when $n=k$, the below holds
\begin{align}
    \epsilon_k \leq \epsilon_1/2^{k-1},
\end{align}
With an odd $K$,
both the largest and the smallest value in set $A_{k+1}$ could be $x_{(K+1)/2}^k$.
In other cases, the largest (smallest) value is $(x_i^k+x_j^k)/2$ ($(x_l^k+x_m^k)/2$), where these values hold
$x_i^k-x_l^k \geq 0$,
$x_m^k-x_j^k \geq 0$, and
$x_i^k-x_k^k + x_l^k-x_j^k \leq \epsilon_k$.
\begin{align}
&\epsilon_{k+1} = \nonumber\\
&\begin{cases}
|(x_i^k-x_l^k)/2 + (x_m^k-x_j^k)/2| \leq (x_i^k-x_l^k + x_m^k-x_j^k)/2 \leq \epsilon_k/2.
& 
    (K\equiv 0\mod 2)
\\
x_{(K+1)/2}^k - (x_l^k+x_m^k)/2 \leq x_l^k - (x_l^k+x_m^k)/2 
\leq (x_l^k-x_m^k)/2 \leq \epsilon_k/2
& 
    (K\equiv 1\mod 2 \text{ and } \max(A_{k+1})=x_{(K+1)/2}^k ) \\
(x_i^k+x_j^k)/2 - x_{(K+1)/2}^k \leq (x_i^k+x_j^k)/2 - x_j^k 
\leq (x_i^k-x_j^k)/2 \leq \epsilon_k/2
& 
    (K\equiv 1\mod 2 \text{ and } \min(A_{k+1})=x_{(K+1)/2}^k )
\end{cases} \nonumber    
\end{align}
\begin{gather}
\therefore
\epsilon_{k+1} \leq \epsilon_1/2^{k}.
\end{gather}
Therefore, we obtain
\begin{align}
    \epsilon_n \leq \epsilon_1/2^{n-1}, ~\lim_{n \to \infty} \epsilon_n &= 0.
\end{align}
This means the minimum difference among $(x_1, x_2, \cdots, x_K)$ resulted in $0$
when $y$ is minimized.
All $x_i$ equal to $K/K=1$. This corresponds to the correlation matrix with no correlation. Therefore the autocorrelation of noise enhances MC than one with no correlation.
\end{proof}

\begin{thm}
The order of $\Msumu$ between different noises.
\\
Considering 2 noises $v_1$ and $v_2$, whose PSDs are $\lambda[\bm{C}_{v_1}]_i$ and $\lambda[\bm{C}_{v_2}]_i$, the sorted series of $\lambda[\bm{C}_{v_1}]_i$  and $\lambda[\bm{C}_{v_2}]_i$ in descending order are newly defined as $\hat{\lambda}[\bm{C}_{v_1}]_i$  and $\hat{\lambda}[\bm{C}_{v_2}]_i$.
When
$\hat{\lambda}[\bm{C}_{v_1}]_i \geq \hat{\lambda}[\bm{C}_{v_2}]_i ~ (1\leq i\leq k_K)$
and 
$\hat{\lambda}[\bm{C}_{v_1}]_i \leq \hat{\lambda}[\bm{C}_{v_2}]_i ~ (k_K+1 \leq i\leq K)$,
$\Msumu$ of each noise, that are ${\Msumu}^{v_1}$ and ${\Msumu}^{v_2}$ indicates the order,
\begin{align}
    {\Msumu}^{v_1} \geq {\Msumu}^{v_2}.\label{order_Csumu_decrease_lam}
\end{align}
\end{thm}
\begin{proof}
To compare 
${\Msumu}^{v_1}=\sum_{i=1}^{N} f(\lambda[\bm{C}_{v_1}]_i)$ and
${\Msumu}^{v_2}=\sum_{i=1}^{N} f(\lambda[\bm{C}_{v_2}]_i)$,
we use
\begin{align}
    f(\lambda_1-d)-f(\lambda_1) \geq f(\lambda_2)-f(\lambda_2+d),\\
    \therefore
    f(\lambda_1) + f(\lambda_2) \leq f(\lambda_1-d) + f(\lambda_2+d),
\end{align}
where $\lambda_1 \leq \lambda_2$ and $0 \leq d <\lambda_1 ~ (d \in \mathbb{R})$.
This property of $f(x)$ allows us to compare summations of $f(\lambda[\bm{C}_{v_1}]_i)$ with respect to $i$, where the average of $\lambda[\bm{C}_{v_1}]_i$ are constant, 
such as $(\lambda_1, \lambda_2)$ and $(\lambda_1-d, \lambda_2+d)$.
Firstly, we choose $\lambda_1=\hat{\lambda}[\bm{C}_{v_2}]_1$ and $\lambda_2=\hat{\lambda}[\bm{C}_{v_2}]_{k_K+1}$.
Using
$\hat{\lambda}[\bm{C}_{v_1}]_1\geq \hat{\lambda}[\bm{C}_{v_2}]_1$ and 
$\hat{\lambda}[\bm{C}_{v_2}]_{k_K+1}\geq \hat{\lambda}[\bm{C}_{v_1}]_{k_K+1}$,

\begin{align}
    f(\hat{\lambda}[\bm{C}_{v_2}]_1) + f(\hat{\lambda}[\bm{C}_{v_2}]_{k_K+1})
    \leq
    f(\hat{\lambda}[\bm{C}_{v_2}]_1 + d) + f(\hat{\lambda}[\bm{C}_{v_2}]_{k_K+1} - d),
\end{align}
if $\hat{\lambda}[\bm{C}_{v_2}]_{k_K+1} - \hat{\lambda}[\bm{C}_{v_1}]_{k_K+1} \leq
\hat{\lambda}[\bm{C}_{v_1}]_1 - \hat{\lambda}[\bm{C}_{v_2}]_1$,
we choose
$d=\hat{\lambda}[\bm{C}_{v_2}]_{k_K+1} - \hat{\lambda}[\bm{C}_{v_1}]_{k_K+1}$, 

\begin{align}
    f(\hat{\lambda}[\bm{C}_{v_2}]_1) + f(\hat{\lambda}[\bm{C}_{v_2}]_{k_K+1}) 
    &\leq
    f(\hat{\lambda}[\bm{C}_{v_2}]_1 + d) + f(\hat{\lambda}[\bm{C}_{v_1}]_{k_K+1}).
\end{align}
While $\sum_{i=1}^{j_1} d_i \leq 
\hat{\lambda}[\bm{C}_{v_1}]_1 - \hat{\lambda}[\bm{C}_{v_2}]_1$ 
where $d_i=\hat{\lambda}[\bm{C}_{v_2}]_{k_K+i} - \hat{\lambda}[\bm{C}_{v_1}]_{k_K+i}$, 
this operation is repeated:

\begin{align}
    f(\hat{\lambda}[\bm{C}_{v_2}]_1) + \sum_{i=1}^{j_1}f(\hat{\lambda}[\bm{C}_{v_2}]_{k_K+i}) 
    &\leq
    f(\hat{\lambda}[\bm{C}_{v_2}]_1 + d) + \sum_{i=1}^{j_1}f(\hat{\lambda}[\bm{C}_{v_1}]_{k_K+i}),
    ~
    d = \sum_{i=1}^{j_1} d_i.
\end{align}
Subsequently, when $\sum_{i=1}^{j_1+1} d_i > 
\hat{\lambda}[\bm{C}_{v_1}]_1 - \hat{\lambda}[\bm{C}_{v_2}]_1$,
we choose
$d=\hat{\lambda}[\bm{C}_{v_1}]_1 - \hat{\lambda}[\bm{C}_{v_2}]_1 - \sum_{i=1}^{j_1} d_i$, 

\begin{align}
    f(\hat{\lambda}[\bm{C}_{v_2}]_1) + \sum_{i=1}^{j_1}f(\hat{\lambda}[\bm{C}_{v_2}]_{k_K+i}) 
    &\leq
    f(\hat{\lambda}[\bm{C}_{v_1}]_1) + \sum_{i=1}^{j_1}f(\hat{\lambda}[\bm{C}_{v_1}]_{k_K+i})
    + f(\hat{\lambda}[\bm{C}_{v_2}]_{k_K +j_1 +1} - d),
\end{align}
While $\sum_{i=1}^{j_1+1} d_i > 
\sum_{i=1}^{j'_1} d'_i$ where $d'_i=\hat{\lambda}[\bm{C}_{v_1}]_i -\hat{\lambda}[\bm{C}_{v_2}]_i$,
this operation is repeated:

\begin{align}
    \sum_{i=1}^{j'_1}f(\hat{\lambda}[\bm{C}_{v_2}]_i) + \sum_{i=1}^{j_1}f(\hat{\lambda}[\bm{C}_{v_2}]_{k_K+i}) 
    &\leq
    \sum_{i=1}^{j'_1}f(\hat{\lambda}[\bm{C}_{v_1}]_i) + \sum_{i=1}^{j_1}f(\hat{\lambda}[\bm{C}_{v_1}]_{k_K+i})
    + f(\hat{\lambda}[\bm{C}_{v_2}]_{k_K +j_1 +1} - d), \label{Eq_meaning_d}
    \\
    d=\sum_{i=1}^{j'_1} d'_i - \sum_{i=1}^{j_1} d_i.
\end{align}
We conduct these operations for $j_i$ and $j'_i$ in the range of $\sum_{i}j_i=K$ and $\sum_{i}j'_i=K$.
The term $f(\hat{\lambda}[\bm{C}_{v_2}]_{k_K +j_1 +1} - d)$ in Eq.~(\ref{Eq_meaning_d}) accumulate the difference between $\sum_{i=1}\hat{\lambda}[\bm{C}_{v_1}]_i$ 
and $\sum_{i=1}\hat{\lambda}[\bm{C}_{v_2}]_i$. 
According to Eq.~(\ref{trace-Cv_sum-PSD}), we know  $\sum_{i=1}^N\hat{\lambda}[\bm{C}_{v_1}]_i = \sum_{i=1}^N\hat{\lambda}[\bm{C}_{v_2}]_i = N$, which ensure that the difference would disappear.
After completing the operations, we obtain
\begin{align}
    \sum_{i=1}^{K}f(\hat{\lambda}[\bm{C}_{v_2}]_1)
    &\leq
    \sum_{i=1}^{k_K-1}f(\hat{\lambda}[\bm{C}_{v_1}]_i ) + \sum_{i=1}^{K-k_K} f(\hat{\lambda}[\bm{C}_{v_1}]_{k_K+i})
    \\
    &+ f(\sum_{i=1}^{K}\hat{\lambda}[\bm{C}_{v_2}]_{k_K+i} 
    - \sum_{i=1}^{K-k_K} \hat{\lambda}[\bm{C}_{v_1}]_{k_K+i}
    - \sum_{i=1}^{k_K-1}\hat{\lambda}[\bm{C}_{v_1}]_i
    )\\
    &=
    \sum_{i=1}^{k_K-1}f(\hat{\lambda}[\bm{C}_{v_1}]_i ) + \sum_{i=1}^{K-k_K} f(\hat{\lambda}[\bm{C}_{v_1}]_{k_K+i} + f(\hat{\lambda}[\bm{C}_{v_1}]_{k_K})
    \\
    &=
    \sum_{i=1}^{K}f(\hat{\lambda}[\bm{C}_{v_1}]_i ).
\end{align}
This proves Eq.~(\ref{order_Csumu_decrease_lam}).
\end{proof}

\section{MF derived from the state expansion of orthogonal basis}
\subsection{Echo state network with linear activation function}
Here, we derive the MF based on the state expansion. In this section, we assume that the input and noise share the common input weight $\bm{w}$.
Using the matrix $\bP$ which is derived from the eigendecomposition $\bW = \bP\bSigma\bP^{-1}$, we define a new variable $\bz_t = \bP^{-1}\bx_t$ and then 
\begin{align}
    \bP^{-1}\bx_{t+1} &= \bSigma \bP^{-1} \bx_t + \bP^{-1} \bm{w} u_{t+1} + \bP^{-1} \bm{w} v_{t+1} \nonumber\\
    \bz_{t+1} &= \bSigma \bz_t + \bP^{-1} \bm{w} u_{t+1} + \bP^{-1} \bm{w} v_{t+1}. \label{eq:y_{t+1}}
\end{align}
Using Eq.~(\ref{eq:y_{t+1}}) repeatedly, we obtain 
\begin{align}
    \bz_{t} &= \sum_{k=0}^\infty (\bSigma^k \bP^{-1} \bm{w} u_{t-k} + \bSigma^k \bP^{-1} \bm{w} v_{t-k}) \\
    \bx_{t} &= \sum_{k=0}^\infty (\bP \bSigma^k \bP^{-1} \bm{w} u_{t-k} + \bP \bSigma^k \bP^{-1} \bm{w} v_{t-k})
\end{align}
Note that, since $\bSigma = {\rm diag}(\lambda_1, \lambda_2, \ldots, \lambda_N)$, $\bSigma^k = {\rm diag}(\lambda_1^k, \lambda_2^k, \ldots, \lambda_N^k)$. 
\begin{align}
    \bX &= \sum_{k=0}^\infty (\bP \bSigma^k \bP^{-1} \bm{w} \bU_k + \bP \bSigma^k \bP^{-1} \bm{w} \bV_k)
\end{align}
Using SVD, we can decompose the state matrix $\bX$ into $\bX=\bPhi\bOmega\bPsi^\top$. 
Accordingly, 
\begin{align}
    \bPsi^\top &= \bOmega^{-1} \bPhi^\top \sum_{k=0}^\infty (\bP \bSigma^k \bP^{-1} \bm{w} \bU_k + \bP \bSigma^k \bP^{-1} \bm{w} \bV_k)
\end{align}
The normalized linearly-independent state $\hat{\bx}_t=[x_{1,t}\cdots x_{N,t}]^\top$ is described by
\begin{align}
    \hat{\bx}_t &= \bOmega^{-1} \bPhi^\top \sum_{k=0}^\infty (\bP \bSigma^k \bP^{-1} \bm{w} u_{t-k} + \bP \bSigma^k \bP^{-1} \bm{w} v_{t-k}).
\end{align}
Note that the inner product is defined by 
\begin{align}
    \left<\by\bz\right>=\sum_{t=1}^T \by_t^\top \bz_t
\end{align}
and the $i$th state time-series vector $\hat{\bx}_{i}=[\hat{x}_{i,1}\cdots\hat{x}_{i,T}]^\top$ satisfies 
\begin{align}
    \left<\hat{\bx}_i \hat{\bx}_j\right> = 
    \begin{cases}
        1 & (i=j) \\
        0 & (i\neq j)
    \end{cases}. 
\end{align}

This separation of $v_t$ enables us to define a new orthogonal basis composing $v$:
\begin{align}
v_t 
&= n_t + \sum_i^{N_a} c_{i,1}\sin{2\pi f_it} + c_{i,2}\cos{2\pi f_it},
\end{align}
where $n_t$ represents the random component, and its delayed series is linearly independent. The left elements compose $a_t$ and have $2N_a$ bases.

\subsection{MF of typical cases}
Here we show the MF with three types of noise: (1) i.i.d. noise, (2) sinusoidal noise, and (3) autocorrelated noise. 

\subsubsection{I.i.d. noise}
Firstly, we assume that $u_t$ is an i.i.d. input with mean of $0$ and $v_t$ is an i.i.d. noise with mean of $0$. 

\begin{align}
    \hat{\bx}_{t} &= 
    \bOmega^{-1} \bPhi^\top \sum_{k=0}^\infty \left( \sqrt{\left<u^2\right>} \bP \bSigma^k \bP^{-1} \bm{w} \hat{u}_{t-k} + \sqrt{\left<v^2\right>} \bP \bSigma^k \bP^{-1} \bm{w} \hat{v}_{t-k} \right), 
\end{align}
where $\hat{u}_{t-k} = u_{t-k}/\sqrt{\left<u^2\right>}$ and $\hat{v}_{t-k} = v_{t-k}/\sqrt{\left<v^2\right>}$ represent the normalized input and noise, respectively $\left(\left<\hat{u}^2\right>=\left<\hat{v}^2\right>=1\right)$.
Since the MF is equivalent to the square norm of the coefficient vector, the MFs with respect to $u_{t-k}$ and $v_{t-k}$ are described by 
\begin{align}
    M(u_{t-k}) 
    &= \left<u^2\right> \left|\left| \bOmega^{-1} \bPhi^\top \bP \bSigma^k \bP^{-1} \bm{w} \right|\right|^2, \\
    M(v_{t-k}) 
    &= \left<v^2\right> \left|\left| \bOmega^{-1} \bPhi^\top \bP \bSigma^k \bP^{-1} \bm{w} \right|\right|^2, 
\end{align}
respectively. 
The MC is described by 
\begin{align} 
    \Msum &= \sum_{k=0}^\infty \left\{ M(u_{t-k}) + M(v_{t-k}) \right\} \nonumber\\
    &= \left( \left<u^2\right> + \left<v^2\right> \right) \sum_{k=0}^\infty \left|\left| \bOmega^{-1} \bPhi^\top \bP \bSigma^k \bP^{-1} \bm{w} \right|\right|^2
\end{align}
and holds the following relation: 
\begin{align}
    \Msum = N. 
\end{align}

\subsubsection{Sinusoidal noise}
The inner product of $\by_{i}=[y_{i,1} ~ \cdots ~ y_{i,T}]^\top$ is defined as 
\begin{align}
    I_{ij} = \by_i^\top \cdot \by_j = \sum_{t=1}^T y_{i,t} y_{j,t}. \label{eq:inner_product}
\end{align}
Bases $\by_{i}$ are orthogonalized such that 
\begin{align}
    I_{ij} = 
    \begin{cases}
        1 & (i=j) \\
        0 & (i\neq j)
    \end{cases}. 
\end{align}
We assume that $v_t = A \sin(\omega t)$. 
According to the inner product in Eq.~(\ref{eq:inner_product}), delayed noise $v_{t-k}$ is composed of only two orthonormal bases of $\hat{v}_{1,t} = \sqrt{2/T} \sin(\omega t)$ and $\hat{v}_{2,t} = \sqrt{2/T} \cos(\omega t)$ because $v_{t-k}$ can be decomposed into 
\begin{align*}
    v_{t-k} &= A \sin(\omega (t-k)) \\
    &= A \cos(\omega k) \sin(\omega t) - A \sin(\omega k) \cos(\omega t) \\
    &= \sqrt{\left<v^2\right>} \cos(\omega k) \hat{v}_{1,t} - \sqrt{\left<v^2\right>} \sin(\omega k) \hat{v}_{2,t}, 
\end{align*}
where $\left<v^2\right> = TA^2/2$. 
As the delayed input series $\{u_{t-k}\}$ and $\{\hat{v}_{1,t}, \hat{v}_{2,t}\}$ are orthogonal from each other, we can reduce the state to 
\begin{align}
    \hat{\bx}_{t} 
    &= \sum_{k=0}^\infty \bOmega^{-1} \bPhi^\top \bP \bSigma^k \bP^{-1} \bm{w} u_{t-k} + \sum_{k=0}^\infty \bOmega^{-1} \bPhi^\top \bP \bSigma^k \bP^{-1} \bm{w} \cdot A \sin(\omega(t - k) ) \nonumber\\
    &= \sum_{k=0}^\infty \left[ \sqrt{\left<u^2\right>} \bOmega^{-1} \bPhi^\top \bP \bSigma^k \bP^{-1} \bm{w} \right] \hat{u}_{t-k} \nonumber\\
    &+ \left[ \sqrt{\left<v^2\right>} \sum_{k=0}^\infty \bOmega^{-1} \bPhi^\top \bP \bSigma^k \bP^{-1} \bm{w} \cos(\omega k) \right] \hat{v}_{1,t} 
    - \left[ \sqrt{\left<v^2\right>} \sum_{k=0}^\infty \bOmega^{-1} \bPhi^\top \bP \bSigma^k \bP^{-1} \bm{w} \sin(\omega k) \right] \hat{v}_{2,t} 
\end{align}
Since the MF is equivalent to the square norm of the coefficient vector, the MFs with respect to $\hat{u}_{t-k}$, $\hat{v}_{1,t}$, and $\hat{v}_{2,t}$ are described by 
\begin{align}
    M(\hat{u}_{t-k}) 
    &= \left<u^2\right> \left|\left| \bOmega^{-1} \bPhi^\top \bP \bSigma^k \bP^{-1} \bm{w} \right|\right|^2, \\
    M(\hat{v}_{1,t}) 
    &= \left<v^2\right> \left|\left| \sum_{k=0}^\infty \bOmega^{-1} \bPhi^\top \bP \bSigma^k \bP^{-1} \bm{w} \cos(\omega k) \right|\right|^2, \\
    M(\hat{v}_{2,t}) 
    &= \left<v^2\right> \left|\left| \sum_{k=0}^\infty \bOmega^{-1} \bPhi^\top \bP \bSigma^k \bP^{-1} \bm{w} \sin(\omega k) \right|\right|^2, 
\end{align}
respectively. 

\subsection{MF of autocorrelated noise}
Finally, we show the MF with general autocorrelated noise.

\subsubsection{Decomposition with current random component}
We assume that an autocorrelated noise $v_t$ can be additively decomposed into a random element $n_t$ and an autocorrelated element $a_t$ as follows: 
\begin{align*}
    v_{t} &= c_n n_{t} + c_a a_{t}, 
\end{align*}
where the two elements have time averages $\left< n_t\right>=\left< 
a_t\right>=0$, and the autocorrelatoin of $n_t$ and $a_t$, and cross correlation between $n_t$ and $a_t$ are 
\begin{align}
    \left< n_{t-i} n_{t-j} \right> 
    = 
    \begin{cases}
        1 & (i=j) \\
        0 & (i\neq j)
    \end{cases}, ~~
    \left< a_{t}, a_{t} \right> = 1, ~~
    \left< n_{t-i}, a_{t-j} \right> = 0, 
\end{align}
respectively. 
We define orthonormal bases $\tilde{a}_{t-k}$ for the autocorrelated elements $a_{t-k}$ using the Gram-Schmidt orthogonalization as 
\begin{align}
    \hat{a}_{t-k} = a_{t-k} - 
    \sum_{i=0}^{k-1} 
    \left< \tilde{a}_{t-i} {a}_{t-k} \right> \tilde{a}_{t-i}, ~~
    \tilde{a}_{t-k} = \frac{\hat{a}_{t-k}}{||\hat{a}_{t-k}||}.
    \label{gram_schmidt_0}
\end{align}
Therefore, $k$-delayed autocorrealted element $a_{t-k}$ is decomposed into the orthonormal basis $\{\tilde{a}_{t-k}\}$ as 
\begin{align*}
    a_{t-k} &= \sum_{i=0}^k \tilde{c}_{ki} \tilde{a}_{t-i}. 
\end{align*}
According to the inner product in Eq.~(\ref{eq:inner_product}), delayed noise $v_{t-k}$ is composed of orthonormal bases of $n_{t-k}$ and $\{a_{t-i}\}_{i=0}^k$ because $v_{t-k}$ can be decomposed into 
\begin{align*}
    v_{t-k} &= c_n n_{t-k} + \sum_{i=0}^k \tilde{c}_{ki} \tilde{a}_{t-i}. 
\end{align*}
\begin{align}
    \left< n_{t-i} n_{t-j} \right> 
    = 
    \begin{cases}
        1 & (i=j) \\
        0 & (i\neq j)
    \end{cases}, ~~
    \left< \tilde{a}_{t-i}  \tilde{a}_{t-j} \right> 
    = 
    \begin{cases}
        1 & (i=j) \\
        0 & (i\neq j)
    \end{cases}, ~~
    \left< n_{t-i}, \tilde{a}_{t-j} \right> = 0, 
\end{align}
As $\{u_{t-k}\}$, $\{n_{t-k}\}$, and $\{\tilde{a}_{t-k}\}$ are orthogonal from each other, we can reduce the state as 
\begin{align}
    \hat{\bx}_{t} 
    &= \sum_{k=0}^\infty \bOmega^{-1} \bPhi^\top \bP \bSigma^k \bP^{-1} \bm{w} u_{t-k} 
    + \sum_{k=0}^\infty \bOmega^{-1} \bPhi^\top \bP \bSigma^k \bP^{-1} \bm{w} v_{t-k} \nonumber\\
    &= \sum_{k=0}^\infty \sqrt{\left<u^2\right>} \bOmega^{-1} \bPhi^\top \bP \bSigma^k \bP^{-1} \bm{w} \hat{u}_{t-k} 
    + c_n \sum_{k=0}^\infty \bOmega^{-1} \bPhi^\top \bP \bSigma^k \bP^{-1} \bm{w} n_{t-k} \nonumber\\
    &+ \sum_{k=0}^\infty \bOmega^{-1} \bPhi^\top \bP \bSigma^k \bP^{-1} \bm{w} \sum_{i=0}^k \tilde{c}_{ki} \tilde{a}_{t-i} 
    \nonumber\\
    &= \sum_{k=0}^\infty \left[ \sqrt{\left<u^2\right>} \bOmega^{-1} \bPhi^\top \bP \bSigma^k \bP^{-1} \bm{w} \right]\hat{u}_{t-k} 
    + \sum_{k=0}^\infty \left[ c_n \bOmega^{-1} \bPhi^\top \bP \bSigma^k \bP^{-1} \bm{w} \right] n_{t-k} \nonumber\\
    &+ \sum_{k=0}^\infty \sum_{i=0}^k
    \bOmega^{-1} \bPhi^\top \bP \bSigma^k \bP^{-1} \bm{w}
    \tilde{c}_{ki} \tilde{a}_{t-i} 
    \nonumber\\
    &= \sum_{k=0}^\infty \left[ \sqrt{\left<u^2\right>} \bOmega^{-1} \bPhi^\top \bP \bSigma^k \bP^{-1} \bm{w} \right]\hat{u}_{t-k} 
    + \sum_{k=0}^\infty \left[ c_n \bOmega^{-1} \bPhi^\top \bP \bSigma^k \bP^{-1} \bm{w} \right] n_{t-k} \nonumber\\
    &+ \sum_{k=0}^\infty \left[ \sum_{i=k}^\infty \tilde{c}_{ik} \bOmega^{-1} \bPhi^\top \bP \bSigma^i \bP^{-1} \bm{w} \right] \tilde{a}_{t-k}. 
\end{align}
The MFs with respect to $\hat{u}_{t-k}$, $n_{t-k}$, and $a_{t-k}$ are described by 
\begin{align}
    M(\hat{u}_{t-k}) 
    &= \left<u^2\right> \left|\left| \bOmega^{-1} \bPhi^\top \bP \bSigma^k \bP^{-1} \bm{w} \right|\right|^2, \\
    M(n_{t-k}) 
    &= c_n^2 \left|\left| \bOmega^{-1} \bPhi^\top \bP \bSigma^k \bP^{-1} \bm{w} \right|\right|^2, \\
    M(a_{t-k}) 
    &= \left|\left| \sum_{i=k}^\infty \tilde{c}_{ik} \bOmega^{-1} \bPhi^\top \bP \bSigma^i \bP^{-1} \bm{w} \right|\right|^2, 
\end{align}
respectively.

\section{The noise with little disturbance effects}
In the main text, we have seen that, with the bases whose number is finite, the noise has little disturbance effects in infinite-dimensional systems, which indicates that the following equation holds:
\begin{align}
\lim_{N\to \infty}\frac{\Msumu}{N} 
= 
\lim_{N\to \infty}\frac{1}{N}\sum_{i=1}^{N}
\frac{1}{ 1 + r \lambda[\bm{C}_v]_i} = 1. \label{def_little_disturb_unlimited_bases}
\end{align}
Moreover, we prove that, even if the noise with the infinite number of bases could show little disturbance under a certain condition. 
According to the completeness property and Eq.~(\ref{MC_diff_large_func}), 
$\Msumv=N-\Msumu=
\sum_{i=1}^{N}
\frac{r \lambda[\bm{C}_v]_i}{ 1 + r \lambda[\bm{C}_v]_i}
$.
Therefore Eq.~(\ref{def_little_disturb_unlimited_bases}) can be replaced by
\begin{align}
\lim_{N\to \infty}\frac{\Msumv}{N} 
= \lim_{N\to \infty} (1-\frac{\Msumu}{N})
\label{eq:little_inhibitory}
\end{align}
Two cases satisfying Eq.~(\ref{eq:little_inhibitory}) exist: (1) $\lim_{N\to\infty} \Msumv < \infty$, (2) $\lim_{N\to\infty} \Msumv/N=\infty$.
For each of them, we show one typical example.

\subsection{$\lim_{N\to\infty} \Msumv < \infty$}
We consider a typical case:
\begin{align}
\sum_{i=1}^\infty
\frac{r \lambda[\bm{C}_v]_i}{ 1 + r \lambda[\bm{C}_v]_i}
= \lim_{N\to \infty}\Msumv
< \infty . \label{ex_enough_converge_condition}
\end{align}
Note that, when Eq.~(\ref{ex_enough_converge_condition}) holds, it is easily proved that the series $\frac{r \lambda[\bm{C}_v]_n}{ 1 + r \lambda[\bm{C}_v]_n}$ should converge to $0$, which is equivalent to
\begin{align}
    \lim_{n\to \infty} \lambda[\bm{C}_v]_{n} =0
\end{align}
Among several types of ratio test for evaluating infinite sum, we select a representing d’Alembert’s ratio test.
Defining the ratio as $L$, we derive the condition in which $\Msumv$ absolutely converges:
\begin{align}
L 
&=
\lim_{n\to \infty}\frac{r \lambda[\bm{C}_v]_{n+1}}{ 1 + r \lambda[\bm{C}_v]_{n+1}}
\frac{ 1 + r \lambda[\bm{C}_v]_{n}}{r \lambda[\bm{C}_v]_{n}}\\
&=
\lim_{n\to \infty}
\frac{\lambda[\bm{C}_v]_{n+1}}{\lambda[\bm{C}_v]_{n}}
\frac{ 1/r + \lambda[\bm{C}_v]_{n}}{ 1/r + \lambda[\bm{C}_v]_{n+1}}\\
&=
\lim_{n\to \infty}
\frac{\lambda[\bm{C}_v]_{n+1}}{\lambda[\bm{C}_v]_{n}}
\frac{ 1/r + 0}{ 1/r + 0}\\
&=
\lim_{n\to \infty}
\frac{\lambda[\bm{C}_v]_{n+1}}{\lambda[\bm{C}_v]_{n}} < 1.
\end{align}
We can confirm that this condition is independent of $r$. 

\subsection{$\lim_{N\to\infty} \Msumv=\infty$}
Here, we consider one example the case of $1/f$-like noise, and prove that the noise satisfies the condition Eq.~(\ref{def_little_disturb_unlimited_bases}). 
The PSD $S_i$ of $1/f$-like noise is described by
\begin{align}
\log{S_i}=-\beta\log{f_i} + \gamma',\\
\lambda[\bm{C}_v]_i
= S_i
= e^{\gamma'}/f_i^{\beta}
= \gamma/f_i^{\beta},
\end{align}
where $f_i=i$ is frequency, and $\gamma=e^{\gamma'}$.
We conduct d’Alembert’s ratio test on $1/f$-like noise,
\begin{align}
L =
\lim_{n\to \infty}
(\frac{\gamma_N}{f_{n+1}^{\beta}})/
(\frac{\gamma_N}{f_{n}^{\beta}})
=
\lim_{n\to \infty}
(\frac{n}{1+n})^\beta
=
\lim_{n\to \infty}
(\frac{1}{1+\frac{1}{n}})^\beta
=1,
\end{align}
which shows that we cannot prove Eq.~(\ref{eq:little_inhibitory}) using this approach.
Therefore, we prove the condition using another method. 
Firstly, we prove the case of $\beta=1$. 
The memory capacities of $u$ and $v$ become 
\begin{align}
\Msumu 
&= 
2 \times \sum_{i=1}^{N/2}\frac{1}{ 1 + r \lambda[\bm{C}_v]_i}
=
2 \times \sum_{i=1}^{N/2}\frac{1}{ 1 + r \frac{\gamma}{f_i^{\beta}}}
=
2 \times \sum_{i=1}^{N/2}\frac{f_i^{\beta}}{ f_i^{\beta} + r \gamma},
\\
\Msumv 
&=
2 \times \sum_{i=1}^{N/2}\frac{ r \gamma}{ f_i^{\beta} + r \gamma}
=
2 \times \sum_{i=1}^{N/2}\frac{ r \gamma}{ i^{\beta} + r \gamma},
\end{align}
respectively.
Considering the condition of Eq.~(\ref{trace-Cv_sum-PSD}), $\sum_{i=1}^{N} \lambda[\bm{C}_v]_i = 2 \sum_{i=1}^{N/2} \gamma f_i^{-\beta}=N$,
we can calculate
$\gamma
=
\frac{N}{2\sum_{i=1}^{N/2} f_i^{-\beta}}
=
\frac{N}{2\sum_{i=1}^{N/2} i^{-\beta}}
$.
We use the following property:
\begin{align}
\int_n^{n+1}f(x)dx < 
f(n) < 
\int_{n-1}^{n}f(x)dx,
\\
\int_1^{1+N/2}f(x)dx <
\sum_{i=1}^{N/2} f(i) < 
f(0) + \int_{1}^{N/2} f(x)dx,\label{limit_sum_int}
\end{align}
where $f(x)$ is a arbitrary monotonically decreasing function.
If $f(x)=x^{-\beta}$, Eq.~(\ref{limit_sum_int}) becomes
\begin{align}
\int_1^{1+N/2}\frac{1}{x^\beta}dx <
\sum_{i=1}^{N/2} i^{-\beta} < 
1+\int_{1}^{N/2}\frac{1}{x^\beta}dx. \label{limit_sum-i_beta_int}
\end{align}

First, we consider the case $\beta=1$, the order of $\Msumv$ is introduced as follows.
To begin, the range of $\gamma$ is computed using Eq.~(\ref{limit_sum-i_beta_int}):
\begin{align}
\log(1+N/2) < \sum_{i=1}^{N/2} i^{-1} < 1+\log(N/2)
\\
\frac{1}{ 1+\log(N/2)} <
\frac{1}{\sum_{i=1}^{N/2} i^{-1}} <
\frac{1}{\log(1+N/2)}
\\
\frac{1}{2}\frac{N}{ 1+\log(N/2)} <
\gamma <
\frac{1}{2}\frac{N}{\log(1+N/2)}
\end{align}
Subsequently, we define two functions $l(i, N)$ and $g(i, N)$ which represent the functions suppressing from below and above 
on each term of the series $\frac{1}{\frac{i^\beta}{r\gamma} + 1 }$ in the summation of $\Msumv$,
\begin{align}
\Msumv
=
2 \times \sum_{i=1}^{N/2} \frac{ 1 }{ i/(r \gamma) +1}
\\
\frac{i}{\frac{r}{2}\frac{N}{\log(1+N/2)}}+1
<
\frac{i}{r\gamma}+1
<
\frac{i}{\frac{r}{2}\frac{N}{ 1+\log(N/2)}}+1
\\
\frac{1}{\frac{i}{\frac{r}{2}\frac{N}{ 1+\log(N/2)}}+1}
<
\frac{1}{\frac{i}{r\gamma}+1}
<
\frac{1}{\frac{i}{\frac{r}{2}\frac{N}{\log(1+N/2)}}+1}
\\
l(i, N) < \frac{1}{\frac{i}{r\gamma}+1} < g(i, N)
\end{align}
Therefore, $l(i, N)$ and $g(i, N)$ can be defined:
\begin{align}
l(i, N) = \frac{1}{\frac{ 2 i (1+\log(N/2))}{rN}+1}, ~~ 
g(i, N) = \frac{1}{\frac{ 2 i \log(1+N/2)}{rN}+1}
\\
2 \sum_{i=1}^{N/2} l(i, N) < \Msumv < 2 \sum_{i=1}^{N/2} g(i, N)
\end{align}
Here, we use Eq.~(\ref{limit_sum_int}) again:
\begin{align}
2\sum_{i=1}^{N/2} g(i, N) 
&=
\sum_{i=1}^{N/2} \frac{1}{\frac{i \log(1+N/2)}{rN} + \frac{1}{2}}
=
\frac{rN}{\log(1+N/2)} \sum_{i=1}^{N/2} \frac{1}{i + \frac{1}{2}\frac{rN}{\log(1+N/2)}}
\\
&<
\frac{rN}{\log(1+N/2)} \left(\frac{1}{\frac{1}{2}\frac{rN}{\log(1+N/2)}} +
\int_{1}^{N/2} \frac{1}{x + \frac{1}{2}\frac{rN}{\log(1+N/2)}} dx
\right)
\\
&=
2 +
\frac{rN}{\log(1+N/2)}
\left[ \log(x + \frac{1}{2}\frac{rN}{\log(1+N/2)}) \right]_{1}^{N/2} 
\\
&=
2 + \frac{rN}{\log(1+N/2)} \left(
\log(\frac{N}{2} + \frac{rN}{2\log(1+N/2)}) 
-\log(1 + \frac{rN}{2\log(1+N/2)})
\right)
\\
&=
2 + \frac{rN}{\log(1+N/2)}
\log\left(
\frac{\frac{N}{2} + \frac{rN}{2\log(1+N/2)}}{1 + \frac{rN}{2\log(1+N/2)}}
\right)
\\
&=
2 + \frac{rN}{\log(1+N/2)}
\log\left(
\frac{1 + \frac{r}{\log(1+N/2)}}{\frac{2}{N} + \frac{r}{\log(1+N/2)}}
\right)
\\
&=
2 + \frac{rN}{\log(1+N/2)}
\log\left(
\log(1+N/2) \times
\frac{1 + \frac{r}{\log(1+N/2)}}{\frac{2\log(1+N/2)}{N} + r}
\right)
\\
&=
2 +
\frac{rN}{\log(1+N/2)} \left(
    \log(\log(1+N/2)) + 
    \log\left(
    \frac{1 + \frac{r}{\log(1+N/2)}}{r + \frac{2\log(1+N/2)}{N}}
    \right)
\right)
\end{align}
\begin{align}
\therefore
\lim_{N\to \infty}\frac{\Msumv}{N}
&<
\lim_{N\to \infty}
\frac{2}{N} +
\frac{r}{\log(1+N/2)} \left(
\log(\log(1+N/2)) + 
    \log\left(
    \frac{1 + \frac{r}{\log(1+N/2)}}{r + \frac{2\log(1+N/2)}{N}}
    \right)
\right)
\\
&=
\lim_{N\to \infty}
0 + \left(
0 \times \log\frac{1+0}{r+0} + \frac{r\log(\log(1+N/2)}{\log(1+N/2)}
\right)
\\
&=
\lim_{N\to \infty}
\frac{r\log(\log(1+N/2)}{\log(1+N/2)}
\\
&=0 ~
\left( \because
\lim_{N\to \infty}
\frac{\log N}{N} 
= 0
\right).
\end{align}
We can see that $\Msumv/N < o(\log(\log N)/\log N)$.
This proves that $1/f$ noise has little inhibitory effects on memory, and this property is independent from the noise intensity $r$.
In addition, because
\begin{align}
2\sum_{i=1}^{N/2}l(i, N) &= 2\sum_{i=1}^{N/2} \frac{1}{\frac{ 2 i (1+\log(N/2))}{rN}+1}\\
&= \frac{1+\log(N/2)}{rN}\sum_{i=1}^{N/2} \frac{1}{i + \frac{1}{2}\frac{rN}{(1+\log(N/2))} }\\
&> \frac{1+\log(N/2)}{rN}\sum_{i=1}^{N/2} \frac{1}{i + \frac{1}{2}\frac{rN}{\log(N/2)} },
\end{align}
we can similarly obtain $o(\log(\log N)/\log N) < \Msumv/N$.
Therefore, $\Msumv/N = o(\log(\log N)/\log N)$.

Next, we prove the case of $\beta>1$.
the line of $\log{S_{i, 1}}=\log{f_i} + \gamma'$ and $\log{S_{i, \beta}}=-\beta\log{f_i} + \gamma'$ has one intersection. The PSD where $f_i$ are smaller (larger) than the intersection holds $S_{i, 1}<S_{i, \beta}$ ($S_{i, 1}>S_{i, \beta}$).
According to Eq.~(\ref{order_Csumu_decrease_lam}), we obtain
\begin{align}
    {\Msumu}^{1} \leq {\Msumu}^{\beta},
\end{align}
where ${\Msumu}^{1}$ and ${\Msumu}^{\beta}$ represent 
the $\Msumu$ of $1/f$ noise and $1/f^\beta$ $(\beta\geq1)$ noise, respectively. This produces
\begin{align}
    N-{\Msumu}^{1} \geq N-{\Msumu}^{\beta},\\
    {\Msumv}^{1} \geq {\Msumv}^{\beta},
\end{align}
where ${\Msumv}^{1}$ and ${\Msumv}^{\beta}$ represent
the $\Msumv$ of $1/f$ noise and $1/f^\beta$ $(\beta\geq1)$ noise, respectively. 
Using $\lim_{N\to \infty} \frac{{\Msumv}^{1}}{N} = 0$ and $\frac{{\Msumv}^{\beta}}{N}\geq 0$,
\begin{align}
    \lim_{N\to \infty} \frac{{\Msumv}^{1}}{N} \geq \lim_{N\to \infty} \frac{{\Msumv}^{\beta}}{N}\geq 0,
    \\
    \lim_{N\to \infty} \frac{{\Msumv}^{\beta}}{N} = 0.
\end{align}
These results prove that $1/f^\beta ~(\beta\geq 1)$ noise has little inhibitory effects on memory within infinite dimensional systems, and this property is independent from the noise intensity $r$.

\section{Further Discussion}
In this study, to evaluate the effects of autocorrelated noise on memory,
we have introduced not only the MF of i.i.d. input but also that of autocorrelated noise.
Considering the noise as an input, this definition of MF enables us to evaluate the memory of any inputs.
We reveled that the number and intensities of linearly independent bases in 
an input determine the amount of memory that the system retains about the input.
This result gives insights into tasks utilizing short-term memory in any linear recurrent networks.
One of those tasks is forecasting the input signals that have not been injected, which has been characterized by the indicator called forecasting capacity (FC)~\cite{gonon2020memory}.
Previous studies have revealed that the autocorrelation of input affects FC, suggesting the relationship between our results.
In the future, we could consider examining how these tasks that exploit short-term memory connect with our findings.


\bibliography{main}